\documentclass[aos,preprint]{imsart}

\RequirePackage{amsthm,amsmath,amsfonts,amssymb}
\RequirePackage[numbers]{natbib}
\RequirePackage[colorlinks,citecolor=blue,urlcolor=blue]{hyperref}
\RequirePackage{graphicx}

\usepackage{natbib}
\usepackage{algorithm,algorithmic}
\usepackage{subcaption}
\usepackage{mathtools}
\usepackage{multirow}
\usepackage{adjustbox}
\usepackage{lscape}
\usepackage{mathtools}
\usepackage{hhline}
\usepackage{subcaption}
\usepackage{longtable}

\startlocaldefs

\newtheorem{theorem}{Theorem}[section]
\newtheorem{lemma}[theorem]{Lemma}
\newtheorem{corollary}[theorem]{Corollary}
\theoremstyle{remark}
\newtheorem{definition}[theorem]{Definition}

\newtheorem*{remark}{Remark}

\newcommand{\nin}{n_{\text{IN}}}
\newcommand{\nout}{n_{\text{OUT}}}

\DeclarePairedDelimiter\floor{\lfloor}{\rfloor}

\endlocaldefs

\begin{document}

\begin{frontmatter}
\title{$V$-statistics and Variance Estimation}
\runtitle{$V$-statistics and Variance Estimation}

\begin{aug}
\author[A]{\fnms{Zhengze} \snm{Zhou}\ead[label=e1]{zz433@cornell.edu}},
\author[B]{\fnms{Lucas} \snm{Mentch}\ead[label=e2]{lkm31@pitt.edu}}
\and
\author[C]{\fnms{Giles} \snm{Hooker}\ead[label=e3]{gjh27@cornell.edu}}
\address[A]{Department of Statistics and Data Science,
Cornell University,
\printead{e1}}

\address[B]{Department of Statistics, 
University of Pittsburgh,
\printead{e2}}
\end{aug}

\address[C]{Department of Statistics and Data Science,
Cornell University and Research School of Finance, Actuarial Studies and Statistics, Australian National University,
\printead{e3}}

\begin{abstract}
This paper develops a general framework for analyzing asymptotics of $V$-statistics. Previous literature on limiting distribution mainly focuses on the cases when $n \to \infty$ with fixed kernel size $k$. Under some regularity conditions, we demonstrate asymptotic normality when $k$ grows with $n$ by utilizing existing results for $U$-statistics. The key in our approach lies in a mathematical reduction to $U$-statistics by designing an equivalent kernel for $V$-statistics. We also provide a unified treatment on variance estimation for both $U$- and $V$-statistics by observing connections to existing methods and proposing an empirically more accurate estimator. Ensemble methods such as random forests, where multiple base learners are trained and aggregated for prediction purposes, serve as a running example throughout the paper because they are a natural and flexible application of $V$-statistics. 
\end{abstract}

\begin{keyword}[class=MSC2010]
\kwd[Primary ]{62G08}
\kwd{62G20}
\kwd[; secondary ]{68T10}
\end{keyword}

\begin{keyword}
\kwd{$V$-statistics}
\kwd{asymptotics}
\kwd{variance estimation}
\kwd{ensemble methods}
\end{keyword}

\end{frontmatter}

\section{Introduction}\label{sec:intro}

This paper develops a general framework for analyzing asymptotics of a class of statistics termed $V$-statistics. A $V$-statistic is a statistical function which is representable as functionals $T\left(F_n\right)$ of the sample distribution $F_n$. Many commonly encountered statistics belong to this category. As an example, the $k^{th}$ central moment is the functional $T\left(F\right) = \int \left(x - E[X]\right)^kdF\left(x\right)$, and the associated sample $k^{th}$ moment yields a $V$-statistic $T\left(F_n\right) = \frac{1}{n}\sum_{i=1}^n\left(x_i - \bar{x}\right)^k$ where $\bar{x}$ is the sample mean. The chi-squared statistic and maximum likelihood estimates can also be viewed as $V$-statistics \cite{MR595165}.

Informally, given data $\mathcal{D} = \{Z_1, \ldots, Z_n\}$ of i.i.d.\ observations drawn from an underlying distribution $F_Z$ and a permutation symmetric kernel function $h_k$ which takes $k$ arguments 
, a $V$-statistic with kernel $h$ of degree $k$ is defined as
$$
    V_n = n^{-k}\sum_{i_1=1}^n\ldots\sum_{i_{k}=1}^nh_{k}\left(Z_{i_1}, \ldots, Z_{i_{k}}\right)
$$
where $\{Z_{i_1}, \ldots, Z_{i_k}\}$ consists of $k$ randomly sampled elements (possibly with duplicates) from $\{Z_1, \ldots, Z_n\}$ and the sum is taken over all ${n^k}$ subsamples of size $k$. This is equivalent to functional $Eh_k(F_n)$ where $F_n$ is the corresponding empirical distribution based on $\{Z_1, \ldots, Z_n\}$.

The asymptotic distribution theory of $V$-statistics was first developed in \cite{MR22330} as $n \to \infty$ with $k$ fixed. Notions of differentiability of $T$ play a key role, where the difference $T(F_n) - T(F)$ can be represented in terms of derivatives of $T(\cdot)$ and the difference $F_n - F$. Asymptotic normality is achieved by studying the first order term and carefully controlling remainders in this form of Taylor expansion. 

$V$-statistics are closed related to $U$-statistics \cite{MR26294}, with a major difference being sampling without replacement when constructing each kernel. A $U$-statistic with a permutation symmetric kernel $h_k$ of degree $k$ is defined as
$$
    U_n = \frac{1}{{n \choose k}}\sum_{i}h_k\left(Z_{i_1}, \ldots, Z_{i_k}\right)
$$
where $\{Z_{i_1}, \ldots, Z_{i_k}\}$ consists of $k$ distinct elements from $\{Z_1, \ldots, Z_n\}$ and the sum is taken over all ${n \choose k}$ subsamples of size $k$. $U$-statistics of this form were shown to be minimum-variance unbiased estimators \cite{MR15746}. Asymptotic normality was demonstrated by \cite{MR26294} using Hoeffding decomposition and H$\acute{a}$jek projections \cite{MR222988}. 

Both $U$- and $V$-statistics tie closely to ensemble methods \cite{MR3184068}, a machine learning paradigm where multiple base learners are trained and aggregated for prediction purposes. An ensemble can usually achieve better generalization performance than individual base learners, with bagged trees \cite{breiman1996bagging} and random forests \cite{breiman2001random} being two successful examples. 

Despite its widespread application, theoretical developments are somewhat limited with many results focusing on the consistency of tree-based ensembles. \cite{MR2447310} provide theorems that establish the universal consistency of averaging rules in classification. \cite{MR3357876} prove a consistency result for random forests in the context of additive regression models. A uniform consistency of random survival forest is introduced in \cite{MR2651045}, and extended in \cite{cui2017consistency}.

Another line of research lies in quantifying sampling uncertainty for ensemble predictions. \cite{MR2654590} consider estimating the standard error of bagged learners and random forests using the jackknife \cite{MR1157715} and bootstrap estimators. \cite{MR3225243} propose applying the jackknife and Infinitesimal Jackknife \cite{MR3265671} for constructing confidence intervals of random forests. However, these results lack a statistical formulation of the asymptotic distribution. 

The asymptotics of $U$-statistics have been revisited recently in the context of ensemble methods \cite{MR3491120, MR3862353} by allowing $k$ to grow with $n$. In particular, \cite{MR3491120} show that under some regularity conditions, ensemble predictions are asymptotically normal and these results are further refined in \cite{peng2019asymptotic}, who also establish rates of convergence. \cite{MR3862353} demonstrate consistency results whenever a carefully designed tree building process is employed.

In this paper, we show that $V$-statistics can be reformulated as $U$-statistics with a well-designed kernel and thus all previous asymptotic results regarding $U$-statistics can be applied. In particular, \cite{MR1294807} demonstrates that a $V$-statistic is asymptotically equivalent to its corresponding $U$-statistic when the order of the kernel grows at the rate of $o(n^{\frac{1}{4}})$. Our construction extends this result to produce a general central limit theorem, which can further be applied to ensembles built by sampling with replacement. 

Another central theme is to accurately and efficiently estimate the limiting variance. In \cite{MR3491120}, the authors demonstrate that the asymptotic variance of a U statistic has an explicit expression and can be estimated at no additional computational cost by an Internal Variance Estimation Method (see Algorithm \ref{algo:IM} in Section \ref{subsec:im}). \cite{MR3862353} apply the Infinitesimal Jackknife developed in \cite{MR3265671} for variance estimation, but without a closed-form expression for the variance term directly.  However, we observe that all existing methods exhibit severe bias when the number of base learners is not sufficiently large. A biased estimate of variance can result in conservative confidence intervals or decrease the power of hypothesis testing \cite{MR3491120, zhou2018approximation}. This phenomenon poses a challenge to any inference tasks as building a large ensemble requires both longer training time and more storage capacity. 

A unified treatment on variance estimation is presented in this paper. We propose an improved version of Internal Variance Estimation Method \cite[IM;][]{MR3491120} named Balanced Variance Estimation Method (BM). BM enjoys lower bias under the same ensemble size. We also point out that these two methods can be reduced to computing the variance of a conditional expectation, a problem which has been studied by several other researchers \cite{MR2844419, MR3600230}. The Infinitesimal Jackknife (IJ) employed in \cite{MR3862353} has a long history in statistical literature \cite{MR659849, MR3265671} and is developed from a different perspective by approximating the sampling distribution. It turns out that a close connection exists between BM and IJ, and we prove their equivalence under a natural condition. To estimate variance in the limiting distribution in finite sample case, we develop a bias-corrected version of BM through an ANOVA-like framework \cite{MR2844419}. The new estimator is shown to produce much more reliable results with moderate numbers of base learners.

Our contributions are summarized as follows: we establish a central limit theorem for infinite order $V$-statistics; we present a unified framework for variance estimation in ensemble methods and design a bias-corrected estimator; thorough empirical studies are conducted to verify our results and demonstrate best practice for doing inference for ensemble methods.

The presentation of this paper connects
$V$-statistics and $U$-statistics closely to ensemble methods, where the underlying kernel $h$ does not need to have an explicit formulation. This abstraction prevents us from taking some specific approaches such as Taylor expansion \cite{MR22330}, but it allows a more general and flexible framework, which can be easily extended to any specific choices of kernel $h$. 

The remaining part of this paper is organized as follows. Section \ref{sec:rl} introduces mathematical notations and related work on $U$-statistics. In Section \ref{sec:v} we analyze $V$-statistics and demonstrate how similar asymptotics can be achieved. Section \ref{sec:ve} is dedicated to a detailed discussion of variance estimation methods. The proposed bias correction algorithm is derived in Section \ref{sec:bc}. Section \ref{sec:re} extends the results to randomized ensembles. Empirical results are displayed in Section \ref{sec:es} and we end with some discussions and future directions in Section \ref{sec:con}. Proofs and additional simulation results are collected in Appendices.

\section{Related Work on $U$-statistics}\label{sec:rl}

We first give a brief introduction on the notion of $U$-statistics, and then illustrate how it can be utilized in the analysis of ensemble models. 

Assume that we have a training set $\mathcal{D} = \{Z_1, \ldots, Z_n\}$ of i.i.d.\ observations of the form $Z_i = (\mathbf{X}_i,Y_i)$ drawn from an underlying distribution $F_Z$, where $\mathbf{X} = (X_1, \ldots, X_p) \in \mathcal{X}$ are $p$ covariates. We want to estimate a parameter of interest $\theta$. Suppose there exists an unbiased estimator $h$ of $\theta$ that is a function of $k \leq n$ arguments (we call $h$ a kernel of size or degree $k$) so that 
$$
\theta = \mathbb{E}h\left(Z_1, \ldots, Z_k\right)
$$
and without loss of generality, assume that $h$ is permutation symmetric in its arguments and $\mathbb{E}h^2\left(Z_1, \ldots, Z_k\right) < \infty$. Then the minimum variance unbiased estimator for $\theta$ is given by 
\begin{equation}\label{eqn:us}
    U_n = \frac{1}{{n \choose k}}\sum_{i}h\left(Z_{i_1}, \ldots, Z_{i_k}\right)
\end{equation}
where $\{Z_{i_1}, \ldots, Z_{i_k}\}$ consists of $k$ distinct elements from the original sample $\{Z_1, \ldots, Z_n\}$ and the sum is taken over all ${n \choose k}$ subsamples of size $k$. The estimator in (\ref{eqn:us}) is referred to as a \emph{complete} $U$-statistic with kernel $h$ of degree $k$. 

There are some natural extensions of (\ref{eqn:us}). To produce more accurate estimations, we would like $k$ to grow with $n$ so the kernel will have access to more information from the data. This results in a kernel that varies with $n$, and an \emph{Infinite Order} $U$-statistic \cite[IOUS;][]{MR1003967}
\begin{equation}\label{eqn:ious}
    U_{n, k_n} = \frac{1}{{n \choose k_n}}\sum_{i} h_{k_n}\left(Z_{i_1}, \ldots, Z_{i_{k_n}}\right).
\end{equation}
Further, evaluating all ${n \choose k_n}$ kernels is computationally infeasible for even moderately sized $n$ or $k_n$ and thus an estimate can be achieved by averaging over only $B_n < {n \choose k_n}$ subsamples. Incorporating this, the estimator becomes an {\em Incomplete Infinite Order} $U$-statistic
\begin{equation}\label{eqn:rus}
    U_{n, k_n, B_n} = \frac{1}{B_n}\sum_{i}h_{k_n}\left(Z_{i_1}, \ldots, Z_{i_{k_n}}\right).
\end{equation}
In (\ref{eqn:ious}) and (\ref{eqn:rus}) we use subscripts to denote that values of $k$ and $B$ may depend on $n$, and the degree of kernel $h$ is $k_n$. 

$U$-statistics of form (\ref{eqn:us}) were first studied in \cite{MR15746} and \cite{MR26294}, where the latter also shows that these statistics are asymptotically normally distributed. \cite{sen1992introduction} provides a review of Hoeffding's seminal paper and outlines the importance of $U$-statistics in modern statistical theory. A comprehensive treatment of the classical $U$-statistics results can be found in \cite{lee1990u} and \cite[Chapter 5]{MR595165}. Certain basic properties, such as almost sure consistency and asymptotic normality, are proved to hold in the case of (\ref{eqn:ious}) and (\ref{eqn:rus}) in \cite{MR1003967}. The connection between $U$-statistics and ensemble methods had not been observed until very recently in the work of \cite{MR3491120} and \cite{MR3862353}.

For simplicity we will focus on the regression setting, where predictions are assumed to be continuous. This would also incorporate binary classification as long as the model predicts the probability instead of majority vote. We are interested in estimating the conditional mean function at a test point $x$
$$
\mu\left(x\right) = \mathbb{E}\left(Y | X = x\right).
$$
Given a base learner $h$, ensemble methods generate resamples $R_1, \ldots, R_B$ of the original data, apply $h$ to each resample, and produce final point estimates by averaging over those generated by each model, yielding estimates of the form
$$
\frac{1}{B}\sum_{i=1}^Bh\left(x;\omega_i, R_i\right).
$$
Here, the $\omega_i$ denotes an auxiliary randomization parameter as used in randomized ensembles like random forests, but which may be dropped for simpler (non-randomized) estimation procedures like bagging\footnote{Sometimes people think of $\omega_i$ as drawing the resamples in bagging and doing both resampling and random feature selection in random forests. Here $\omega$ has a different meaning as we write out the resamples explicitly.}; \cite{peng2019asymptotic} refers to these estimators as generalized $U$-statistics. When all instances of the randomness are considered, note that the kernel again becomes nonrandom, as in \cite{MR3862353} where the authors assume $B$ is large enough for Monte Carlo effects not to matter.

The conventional procedure in random forests is to take $R_1, \ldots, R_B$ to be bootstrap samples, which turns out very difficult to analyze statistically. \cite{MR3491120} propose the following procedure to construct an ensemble. Given a training set $\mathcal{D}$ of size $n$, an ensemble consisting of $B_n$ base learners is constructed using subsamples of size $k_n$
\begin{equation}\label{eqn:ens_U}
    U_{n, k_n, B_n}\left(x\right) = \frac{1}{B_n}\sum_{i=1}^{B_n}h_{k_n}\left(x;Z^*_{i_1}, \ldots, Z^*_{i_{k_n}}\right)
\end{equation}
where $\{Z^*_{i_1}, \ldots, Z^*_{i_{k_n}}\}$ is drawn \emph{without} replacement from $\{Z_1, \ldots, Z_n\}$. This fits into the statistical framework of $U$-statistics and asymptotic normality can be demonstrated under some regularity conditions (see \cite{peng2019asymptotic} for refined results). In particular, the explicit expression for the variance of predictions at any given point can be written in closed-form
\begin{equation}
\label{eqn:var}
\frac{k_{n}^{2}}{n} \zeta_{1,k_n} + \frac{1}{B_n} \zeta_{k_n,k_n}.
\end{equation}
For a given $c$, $1 \leq c \leq k_n$, the variance parameters are defined as 
\begin{equation}\label{eqn:zetac}
    \zeta_{c,k_n} = \text{cov} \left(h_{k_n}(Z_1, \ldots, Z_{k_n}), h_{k_n}(Z_1, \ldots, Z_c,Z_{c+1}^{'}, \ldots, Z_{k_n}^{'}) \right) 
\end{equation}
where $Z_{c+1}^{'}, \ldots, Z_{k_n}^{'}$ are i.i.d.\ copies from the same distribution $F_Z$ and independent of the original data $Z_1, \ldots, Z_n$. For notational simplicity, we drop the test point $x$ in (\ref{eqn:zetac}).

We should mention that the asymptotic distribution is centered at $\theta_k = \mathbb{E}h_{k_n}(Z_1, \ldots, Z_{k_n})$ instead of the true conditional mean $\mathbb{E}(Y | X = x)$. This means that any inferential statements must, in general, be made about the sampling structure of the ensemble rather than the underlying data generating process.  A careful analysis of specific choices of the base learner $h$ and the relationship between covariates $X$ and response $Y$ are central in achieving consistent predictions and is not the focus of this paper.  Some work along these lines includes \cite{MR3862353} which focus on particular tree-building methods, and \cite{MR3357876} which demonstrate the $\mathbb{L}^2$ consistency for random forests when the underlying response corresponds to an additive regression model.

Also note that the asymptotic normality result in \cite{MR3862353} can be viewed as a special case of (\ref{eqn:var}). Here, the authors assume that ensemble size $B$ is large enough for Monte Carlo effects not to matter, in which case (\ref{eqn:var}) reduces to $\frac{k_n^2}{n}\zeta_{1,k_n}$. 

\section{$V$-statistics}\label{sec:v}

$V$-statistics are closely related to $U$-statistics except that the data used in each kernel is sampled \emph{with} replacement. Similar to (\ref{eqn:us}), a \emph{complete} $V$-statistic with kernel $h$ of degree $k$ is defined as
\begin{equation}\label{eqn:vs}
    V_n = n^{-k}\sum_{i_i=1}^n\ldots\sum_{i_{k}=1}^nh_{k}\left(Z_{i_1}, \ldots, Z_{i_{k}}\right)
\end{equation}
where $\{Z_{i_1}, \ldots, Z_{i_k}\}$ consists of $k$ elements from $\{Z_1, \ldots, Z_n\}$ and the sum is taken over all ${n^k}$ subsamples of size $k$.  An \emph{Infinite Order} $V$-statistic (IOVS) is defined analogously to (\ref{eqn:ious})
\begin{equation}\label{eqn:iovs}
    V_{n, k_n} = n^{-k_n}\sum_{i_i=1}^n\ldots\sum_{i_{k_n}=1}^nh_{k_n}\left(Z_{i_1}, \ldots, Z_{i_{k_n}}\right).
\end{equation}

\subsection{Asymptotic Equivalence to $U$-statistics}
We first show that the asymptotic behavior of $V_{n, k_n}$ is the same as that of $U_{n, k_n}$, provided $k_n = o(n^{\frac{1}{4}})$. The following important lemma relates $V_{n, k_n}$ to a family of $U$-statistics, which is a simple extension from Theorem 1 in \cite[p.183]{lee1990u} to the case where the kernel size $k_n$ is changing with $n$.

\begin{lemma}\label{lm:vtou}
Let $V_{n, k_n}$ be a complete, infinite order $V$-statistic based on a permutation symmetric kernel $h_{k_n}$ of degree $k_n$ as defined in (\ref{eqn:iovs}).
Then we may write 
$$
V_{n, k_n} = n^{-k_n}\sum_{j=1}^{k_n}j!S_{k_n}^{(j)}{n \choose j}U_{n}^{(j)}
$$
where $U_{n}^{(j)}$ is a $U$-statistic of degree $j$. The kernel $\phi_{(j)}$ of $U_{n}^{(j)}$ is given by
$$
\phi_{(j)}\left(z_1, \ldots, z_j\right) = \left(j!S_{k_n}^{(j)}\right)^{-1}\sum\nolimits_{(j)}^*h_{k_n}\left(z_{i_1}, \ldots, z_{i_{k_n}}\right)
$$
where the sum $\sum\nolimits_{(j)}^*$ is taken over all $k_n$-tuples $(i_1, \ldots, i_{k_n})$ formed from $\{1, 2, \ldots, j\}$ having exactly $j$ indices distinct, and where the quantities $S_{k_n}^{(j)}$ are Stirling numbers of the second kind \cite{MR241310}.
\end{lemma}

Intuitively, as $n$ grows, if $k_n$ grows slowly enough, $V_{n, k_n}$ should behave like $U_{n, k_n}$, as the difference brought by sampling \emph{with} or \emph{without} replacement becomes negligible. Theorem \ref{thm:vandu} extends a result in \cite{MR1294807} which makes this argument rigorous. 

\begin{theorem}\label{thm:vandu}
Suppose $k_n = o(n^{\frac{1}{4}})$ and $\lim_{n \to \infty} \text{Var}(\sqrt{n}U_{n, k_n}) > 0$. Then $V_{n, k_n}$ and $U_{n, k_n}$ have the same asymptotic distribution.
\end{theorem}

\begin{remark}
The assumption $\lim_{n \to \infty} \text{Var}(\sqrt{n}U_{n, k_n}) > 0$ simply indicates that the rate of convergence for $U_{n, k_n}$ is $\sqrt{n}$. Theorem \ref{thm:vandu} may possibly hold under other regimes, such as with degenerate kernels,  where the convergence rate is not $\sqrt{n}$, but this is out of the scope of this paper.
\end{remark}


As in Equation (\ref{eqn:rus}), by averaging only $B_n < n^{k_n}$ set of subsamples we have an \emph{incomplete, infinite order} $V$-statistic
\begin{equation}\label{eqn:rvs}
    V_{n, k_n, B_n} = \frac{1}{B_n}\sum_{i}h_{k_n}\left(Z_{i_1}, \ldots, Z_{i_{k_n}}\right)
\end{equation}
where $\{Z_{i_1}, \ldots, Z_{i_k}\}$ is again drawn \emph{with} replacement from $\{Z_1, \ldots, Z_n\}$. Under some regularity conditions, similar asymptotic results as Theorem 1 in \cite{MR3491120,peng2019asymptotic} can be shown.

\begin{theorem}\label{thm:v}
Let $Z_1, Z_2, \ldots, Z_n\stackrel{iid}{\sim} F_Z$ and let $V_{n, k_n, B_n}$ be an incomplete, infinite order $V$-statistic with kernel $h_{k_n}$. Let $\theta_{k_n} = \mathbb{E}h_{k_n}(Z_1, \ldots, Z_{k_n})$ such that $Eh_{k_n}^2(Z_1, \ldots, Z_{k_n}) < \infty$. Then under the assumptions that $k_n = o(n^{\frac{1}{4}})$, $\lim_{n \to \infty}k_n^2 \zeta_{1, k_n} > 0$ and $\lim_{n \to \infty} \frac{\zeta_{k_n, k_n}}{n\zeta_{1, k_n}} \to 0$, we have
$$
\frac{\left(V_{n, k_n, B_n} - \theta_{k_n}\right)}{\sqrt{\frac{k_n^2}{n}\zeta_{1, k_n} + \frac{1}{B_n}\zeta_{k_n, k_n}}}\stackrel{d}{\to} \mathcal{N}(0, 1).
$$
In the complete case where $B_n = n^{k_n}$, we have
$$
\frac{(V_{n, k_n} - \theta_{k_n})}{\sqrt{\frac{k_n^2}{n}\zeta_{1, k_n} }}\stackrel{d}{\to} \mathcal{N}(0, 1).
$$

\end{theorem}
Note that the first two assumptions $k_n = o(n^{\frac{1}{4}})$ and $\lim_{n \to \infty}k_n^2 \zeta_{1, k_n} > 0$ ensure that we can apply Theorem \ref{thm:vandu}. The proof requires an additional lemma and is collected together in Appendix \ref{p:thm:v}.

\subsection{Representation As $U$-statistics}\label{sec:vasu}

This section develops a broader connection between $V$- and $U$-statistics to show that the former automatically achieve almost all the properties of the latter.
A complete, infinite order $V$-statistic $V_{n, k_n}$ with kernel $h_{k_n}$ can be written as a corresponding $U$-statistic but with a more complicated kernel derived from $h_{k_n}(\cdot)$. 

Let $\Omega$ denote the set of $\{1, 2, \ldots, n\}$. We use $B^{k_n}(\Omega)$ to denote all size $k_n$ permutations of $\Omega$ with replacement, and let $S^{k_n}(\Omega)$ denote subsamples of size $k_n$ without replacement so that $|B^{k_n}(\Omega)| = n^{k_n}$ and $|S^{k_n}(\Omega)| = {n \choose k_n}$. We can write $V_{n, k_n}$ as 
$$
V_{n, k_n} = n^{-k_n}\sum_{b \in B^{k_n}(\Omega)}h_{k_n}\left(Z_b\right)
$$
where $b$ has $k_n$ elements and $Z_b$ are those $Z$'s with index in $b$.

Equivalently, $V_{n, k_n}$ can be expressed as
$$
V_{n, k_n} = n^{-k_n}\sum_{s \in S^{k_n}(\Omega)}\phi_{k_n}\left(Z_s\right)
$$
with composite kernel $\phi_{k_n}$ defined by
$$
\phi_{k_n}(Z_s) = \sum_{b \in B^{k_n}(s)}\omega_bh_{k_n}\left(Z_b\right)
$$
where $\omega_b$ is the weight associated with each evaluation of $h_{k_n}$ to account for the multiplicity in sampling the same $b$ from $B^{k_n}(s)$ for different $s$.  

For $b = \{i_1, i_2, \ldots, i_{k_n}\}$, we use $u(b) \in \{1, 2, \ldots, k_n \}$ to denote the number of unique elements in $b$ and we have 
$$
\omega_b = \frac{1}{{n-u(b) \choose k_n - u(b)}}.
$$
In practice, it may be computationally infeasible to evaluate $\phi_{k_n}$ explicitly, particularly when $k_n$ is large and therefore so is the collection of  $b \in B^{k_n}(s)$.


We can express $V_{n, k_n}$ as a $U$-statistic
$$
V_{n, k_n} = \frac{1}{{n \choose k_n}}\sum_{s \in S^{k_n}(\Omega)}\phi_{k_n}^*\left(Z_s\right)
$$
where the kernel $\phi_{k_n}^*$ is defined as
$$
\phi_{k_n}^*(Z_s) =\frac{{n \choose k_n}}{n^{k_n}} \sum_{b \in B^{k_n}(s)}\omega_bh_{k_n}\left(Z_b\right).
$$
Here $\omega_b$ is defined as before, and 
$$
\sum_{b \in B^{k_n}(s)}\omega_b = \frac{n^{k_n}}{{n \choose k_n}}.
$$

A more general result for the aymptotics of $V$-statistics is stated in the following theorem, where the restriction that $k_n = o(n^{\frac{1}{4}})$ in Theorem \ref{thm:v} can be discarded.

\begin{theorem}\label{thm:v2}
Let $Z_1, Z_2, \ldots, Z_n \stackrel{iid}{\sim} F_Z$ and let $V_{n, k_n, B_n}$ be an incomplete, infinite order $V$-statistic with kernel $h_{k_n}$. Let $\theta_{k_n} = \mathbb{E}h_{k_n}(Z_1, \ldots, Z_{k_n})$ such that $Eh_{k_n}^2(Z_1, \ldots, Z_{k_n}) < \infty$. Then under the assumption that $\lim \frac{\zeta^*_{k_n, k_n}}{n\zeta^*_{1, k_n}} \to 0$, we have
$$
\frac{\left(V_{n, k_n, B_n} - \theta_{k_n}\right)}{\sqrt{\frac{k_n^2}{n}\zeta^*_{1, k_n} + \frac{1}{B_n}\zeta_{k_n, k_n}}}\stackrel{d}{\to} \mathcal{N}(0, 1).
$$
In the complete case where $B_n = n^{k_n}$, we have
$$
\frac{\left(V_{n, k_n} - \theta_{k_n}\right)}{\sqrt{\frac{k_n^2}{n}\zeta^*_{1, k_n} }}\stackrel{d}{\to} \mathcal{N}(0, 1).
$$

Here, the variance parameter $\zeta^*_{1,k_n}$ is defined as in Equation (\ref{eqn:zetac}) by replacing kernel $h_{k_n}$ with $\phi_{k_n}^*$
$$
    \zeta^*_{1,k_n} = \text{cov} \left(\phi_{k_n}^*\left(Z_1, \ldots, Z_{k_n}\right), \phi_{k_n}^*\left(Z_1, Z_{2}^{'}, \ldots, Z_{k_n}^{'}\right) \right) 
$$
and $\zeta_{k_n,k_n} = \text{var} \left(h_{k_n}\left(Z_1, \ldots, Z_{k_n}\right) \right)$ is still the variance across individual kernels $h_{k_n}$.
\end{theorem}
This theorem provides a more general result for the asymptotics of $V$-statistics. It is essentially a reduction to $U$-statistics by constructing a new kernel representation. The variance expression $\frac{k_n^2}{n}\zeta^*_{1, k_n} + \frac{1}{B_n}\zeta_{k_n, k_n}$ again can be viewed as two parts: the first part $\frac{k_n^2}{n}\zeta^*_{1, k_n}$ comes from the complete case; the second part $\frac{1}{B_n}\zeta_{k_n, k_n}$ is the additional Monte Carlo variance introduced due to incomplete case, which is why $\zeta_{k_n, k_n}$ only involves the original kernel $h_{k_n}$ instead of the composite kernel $\phi_{k_n}^*$. \cite{peng2019asymptotic} provides a unified analysis of these two components by incorporating the choice of subsamples into the randomization parameters of the generalized $U$-statistic. This strategy is not available in our case, as detailed in Section \ref{sec:re}.

The introduction of new kernel $\phi_{k_n}^*$ facilitates theoretical analysis, but it brings challenges in estimating variance component $\zeta_{1, k_n}^*$ directly: it is not feasible to calculate $\phi_{k_n}^*(Z_s)$ for any $s \in S^{k_n}(\Omega)$. We will see in Section \ref{sec:ve} that as a general variance estimation method, the Infinitesimal Jackknife (IM) can be applied. And based on Theorem \ref{thm:bm2ij}, Balanced Variance Estimation Method is equivalently valid without resorting to $\phi_{k_n}^*(Z_s)$ directly.

\section{Variance Estimation}\label{sec:ve}

This section addresses how to estimate variance in the limiting distribution. \cite{MR3491120} propose Internal Variance Estimation Method (IM) based on a two-level sampling procedure. Inspired from this, we design the Balanced Variance Estimation Method (BM) which is shown to have lower bias compared to IM. Unlike IM and BM, the Infinitesimal Jackknife (IJ) employed in \cite{MR3862353} does not depend on an explicit expression for the variance term. All methods presented can apply to both $U$- and $V$-statistics, though they exhibit different performances when sampling without or with replacement, especially in terms of bias. 

IM and BM operate by directly estimating $\zeta_{1,k_n}$ and $\zeta_{k_n,k_n}$ as defined in (\ref{eqn:zetac}). Notice that $\zeta_{k_n,k_n} = \text{var}(h_{k_n}(Z_1, \ldots, Z_{k_n}))$, which can be simply estimated as the variance across all base learners. The estimation for $\zeta_{1,k_n}$ is much more involved. The sample covariance between predictions may serve as a consistent estimator, but in practice it is numerically unstable and often results in negative variance estimates \cite{MR3491120}. Thus we work with the equivalent expression for $\zeta_{1,k_n}$ \cite{lee1990u}
\begin{equation}\label{eqn:c_zeta1}
    \zeta_{1,k_n} = \text{var} \left(\mathbb{E}\left(h_{k_n}(Z_1, \ldots, Z_{k_n}\right)|Z_1 = z_1)\right).
\end{equation}

Expressions of the form from (\ref{eqn:c_zeta1}) belong to an important theme in statistics: estimating the variance of a conditional expectation. It is usually related to uncertainty quantification and has been studied intensively in a number of fields \cite{zouaoui2003accounting, MR2743886}. For a more detailed review, we refer readers to \cite{MR2844419}.

In what follows, assume we have data $\mathcal{D} = \{Z_1, \ldots, Z_n\}$ of i.i.d.\ observations of the form $Z_i = (\mathbf{X}_i,Y_i)$, and a kernel function $h_{k_n}(Z_1, \ldots, Z_{k_n})$. For simplicity, we suppress notations by dropping the test point $x$ in the kernel expression.

\subsection{Internal Variance Estimation Method}\label{subsec:im}

IM was first proposed in \cite{MR3491120} wherein the estimates are obtained as a result of restructuring the ensemble building procedure. It can be viewed as a nested two-level Monte Carlo, where we need to choose $\nout$ and $\nin$ for the number of outer and inner iterations respectively. See Algorithm \ref{algo:IM} for details.

\begin{algorithm}
\begin{algorithmic}

\FOR {$i$ in 1 to $\nout$}

\STATE Select initial fixed point $\tilde{\mathbf{z}}^{(i)}$

\FOR {$j$ in 1 to $\nin$}

	\STATE Select subsample $\mathcal{S}_{\tilde{\mathbf{z}}^{(i)},j}$ of size $k_n$ from training set that includes $\tilde{\mathbf{z}}^{(i)}$
	\STATE Build base learner and evaluate $h_{k_n}(\mathcal{S}_{\tilde{\mathbf{z}}^{(i)},j})$
\ENDFOR

\STATE Record average of the $\nin$ predictions 

\ENDFOR

\STATE Compute the variance of the $\nout$ averages to estimate $\zeta_{1,k_n}$
\STATE Compute the variance of all predictions to estimate $\zeta_{k_n,k_n}$
\STATE Compute the mean of all predictions to obtain final ensemble prediction
\end{algorithmic}
\caption{Internal Variance Estimation Method}
\label{algo:IM}
\end{algorithm}

We use the shorthand $h_{i, j}$ to denote $h_{k_n}(\mathcal{S}_{\tilde{\mathbf{z}}^{(i)},j})$. The average across inner level is calculated as $\bar{h}_i = \frac{1}{\nin}\sum_{j=1}^{\nin}h_{i,j}$. Further we use $\bar{h} = \sum_{i=1}^{\nout}\bar{h}_i$ to denote the average across outer level $i$. Then the estimates for $\zeta_{1,k_n}$ and $\zeta_{k_n,k_n}$ can be expressed as
$$
\hat{\zeta}_{1,k_n}^{\text{IM}} = \frac{1}{\nout - 1}\sum_{i=1}^{\nout}\left(\bar{h}_i - \bar{h}\right)^2
$$
and
$$
\hat{\zeta}_{k_n,k_n}^{\text{IM}} = \frac{1}{\nin \times \nout - 1} \sum_{i=1}^{\nout} \sum_{j=1}^{\nin} \left(h_{i,j} - \bar{h}\right)^2.
$$

\subsection{Balanced Variance Estimation Method}\label{sec:BM}

As show in Figure \ref{fig:ve}, the estimator for $\hat{\zeta}_{1,k_n}^{\text{IM}}$ given by IM is severely biased upwards when $\nin$ and $\nout$ are not sufficiently large ($B_n = \nin \times \nout$). IM is not optimal in the sense that it does not utilize all the information in the ensemble. In particular, $h_{i, j}$ is only used once in the outer iteration $i$ when conditioned on $\tilde{\mathbf{z}}^{(i)}$. Ideally we could also utilize $h_{i, j}$ by conditioning on the remaining $k_n - 1$ inputs. Further, we need to choose two hyperparameters $\nout$ and $\nin$ instead of fixing the number of base learners $B_n$. It is not clear what combination will yield optimal performance under the same computational budget this trade-off will likely differ depending on whether we wish to optimize predictive performance or variance estimation.

\begin{algorithm}
\begin{algorithmic}

\FOR {$b$ in 1 to $B_n$}

\STATE Select subsample $\mathcal{S}_b$ of size $k_n$ from training set $\mathcal{D}$ of size $n$.
\STATE Build base learner and evaluate $h_{k_n}(\mathcal{S}_b)$

\ENDFOR

\FOR {$i$ in 1 to $n$}

\STATE Calculate $m_i$ as the average of $h_{k_n}(\mathcal{S}_b)$ where the $i^{th}$ training sample appears in $\mathcal{S}_b$, weighted by the number of appearance.

\ENDFOR

\STATE Compute the variance of $m_i$ to estimate $\zeta_{1,k_n}$
\STATE Compute the variance of all predictions to estimate $\zeta_{k_n,k_n}$
\STATE Compute the mean of all predictions to obtain final ensemble prediction
\end{algorithmic}
\caption{Balanced Variance Estimation Method}
\label{algo:BM}
\end{algorithm}

To address these issues, we design the Balanced Variance Estimation Method (Algorithm \ref{algo:BM}). In the following, we use $h_b$ to represent $h_{k_n}(\mathcal{S}_b)$ if there is no ambiguity. Let $N_{i, b}$ denote the number of times $i^{th}$ training sample appears in subsample $\mathcal{S}_b$. Summing over $b$ gives $N_i = \sum_{b=1}^{B_n} N_{i,b}$ and the averaged version $\bar{N}_i = \frac{N_i}{B_n}$. For $1 \leq i \leq n$, define
$$
m_i = \sum_{b=1}^{B_n}\omega_{i, b}h_b
$$
where $\omega_{i, b} = \frac{N_{i, b}}{N_i}$. Further define the average of $m_i$ as $\bar{m} = \frac{1}{n}\sum_{i=1}^nm_i$ and the overall average of $h_b$ as $\bar{h} = \sum_{b=1}^{B_n}h_b$.
The estimates for $\zeta_{1,k_n}$ and $\zeta_{k_n,k_n}$ can be written as
$$
\hat{\zeta}_{1,k_n}^{\text{BM}} = \frac{1}{n - 1}\sum_{i=1}^{n}\left(m_i - \bar{m}\right)^2,
$$
and
$$
\hat{\zeta}_{k_n,k_n}^{\text{BM}} = \frac{1}{B_n - 1} \sum_{b=1}^{B_n}\left(h_b - \bar{h}\right)^2.
$$

\subsection{Infinitesimal Jackknife}\label{sec:IJ}

The Infinitesimal Jackknife (IJ) was first studied by \cite{jaeckel1972infinitesimal} as an extension for the jackknife to estimate variance. The basic idea of the jackknife is to omit one observation and
recompute the estimate using the remaining samples. Alternatively, if we assign a weight to each observation, omitting one is equivalent to setting the corresponding weight to zero. More generally, we can give each observation a weight slightly less than one every time. IJ is the limiting case as this deficiency in the weight approaches zero. \cite{MR659849} provided a more detailed treatment of these resampling plans. More recently, IJ was found to be a powerful tool for estimating standard errors in bagging \cite{MR3265671}. \cite{MR3225243} and \cite{MR3862353} applied IJ in the context of random forests. 

In our setting, the Infinitesimal Jackknife estimate of variance can be expressed as
$$
\hat{V}_{\text{IJ}} = \sum_{i=1}^n \text{cov}^2\left(N_{i, b}, h_b\right)
$$
where $\text{cov}(N_{i, b}, h_b) = \frac{\sum_{b=1}^{B_n}(N_{i, b} - \bar{N}_i)(h_b - \bar{h})}{B_n}$ and $\bar{N}_i$ is defined in Section \ref{sec:BM}.

IJ does not rely on an explicit expression of the variance term and is targeted at estimating the limiting variance assuming $B_n$ is sufficiently large. That is, $\hat{V}_{\text{IJ}}$ is essentially estimating $\frac{k_n^2}{n}\zeta_{1,k_n}$. A direct connection exists between BM and IJ, which we will show below. 

\begin{definition}{Balanced Subsample Structure}

We call a subsample structure balanced if $B_n \times k_n$ is a multiple of $n$, and each training sample appears exactly $r_n = \frac{B_n \times k_n}{n}$ times. 

\end{definition}

For $U$-statistics, this structure implies that each training observation appears in exactly $r_n$ base learners. For $V$-statistics, each sample is required to occur $r_n$ times but may be used in fewer than $r_n$ base learners since the sampling is done with replacement. 

\begin{theorem}\label{thm:bm2ij}

If we have balanced subsample structure, the Balanced Variance Estimation Method and the Infinitesimal Jackknife estimator satisfy
$$
\frac{k_n^2}{n}\hat{\zeta}_{1,k_n}^{\text{BM}} = \frac{n}{n - 1}\hat{V}_{\text{IJ}}.
$$
\end{theorem}

\begin{remark}
The scaling factor $\frac{n}{n - 1}$ is a result of how we calculate the empirical variance. If instead we define $\hat{\zeta}_{1,k_n}^{\text{BM}} = \frac{1}{n}\sum_{i=1}^{n}(m_i - \bar{m})^2$, then the two estimators are equal: $\frac{k_n^2}{n}\hat{\zeta}_{1,k_n}^{\text{BM}} = \hat{V}_{\text{IJ}}$.
\end{remark}

\section{Bias Corrections for Variance Estimates}\label{sec:bc}  

\subsection{Bias in Variance Estimation}

As briefly mentioned in Section \ref{sec:intro}, all existing variance estimation methods exhibit severe bias when the number of base learners is not sufficiently large. We now conduct a simple simulation to demonstrate the extent of this bias. Suppose $X \sim 20 \times \text{unif}(0, 1)$ and $Y = 2X + \mathcal{N}(0, 1)$. An ensemble of decision trees is built to predict $Y$ from $X$, and we calculate the variance of the prediction at $x = 10$ using IM, BM and IJ. In our simulation, we fix the number of training observations $n = 500$ and kernel size $k_n = 100$. The number of base learners $B_n$ is varied among $100, 1000$ and $10000$. 

Figure \ref{fig:ve} shows the result for both $U$- and $V$-statistics. Notice that although larger $B_n$ indicates lower variance (see Equation (\ref{eqn:var})), the influence is almost negligible as the dominating part is $\frac{k_n^2}{n}\zeta_{1, k_n}$ in our case (compare Figure \ref{fig:zeta1} with \ref{fig:zetan} in Appendix \ref{app:add}). In order to provide a fair comparison, the variance shown in the figure is for $\frac{k_n^2}{n}\zeta_{1, k_n} + \frac{1}{1000}\zeta_{k_n, k_n}$. Different values of $B_n$ only have effect on the estimation for $\zeta_{1, k_n}$ and $\zeta_{k_n, k_n}$. 

We can easily observe that all three methods (IM, BM, IJ) badly overestimate the variance (notice the log scale on y-axis). The bias mainly arises from an overestimation of $\zeta_{1, k_n}$ (see Figure \ref{fig:zeta1} and \ref{fig:zetan}). However, BM and IJ are better than IM since they utilize more information. The plot also corroborates Theorem \ref{thm:bm2ij}: BM and IJ are exactly the same up to a scaling factor. 

\begin{figure}
\centering
\begin{subfigure}{\textwidth}
\includegraphics[width=\textwidth]{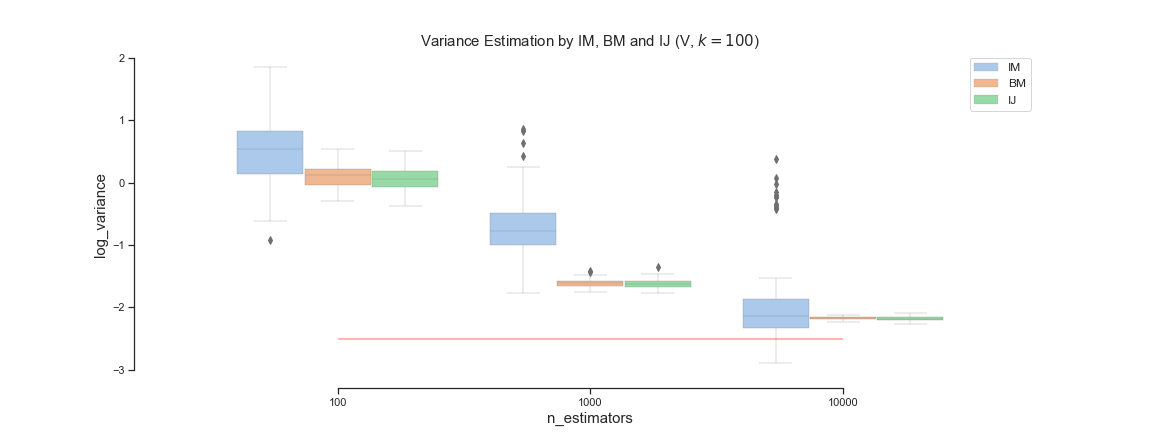}
\caption{Subsampling with replacement ($V$-statistics).}
\label{fig:ve_v}
\end{subfigure}
\begin{subfigure}{\textwidth}
\includegraphics[width=\textwidth]{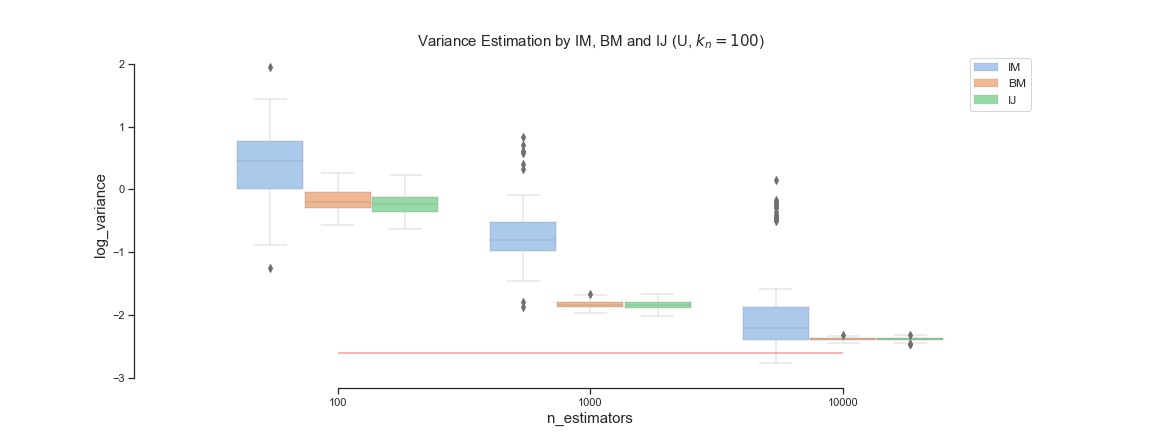}
\caption{Subsampling with replacement ($U$-statistics).}
\label{fig:ve_u}
\end{subfigure}

 \caption{Variance estimation by three different methods: Internal Variance Estimation Method (IM), Balanced Variance Estimation Method (BM) and Infinitesimal Jackknife (IJ). The red line denotes true (log) variance obtained by generating data, training the ensemble 100 times and calculating the empirical variance of predictions. }
\label{fig:ve}
\end{figure}

In Appendix \ref{app:add}, we include additional simulation results on the effect of different kernel size $k_n$. It is worth noting that the pattern of bias is consistent among $V$-statistics: an overestimation of $\zeta_{1, k_n}$ leads to severe bias which diminishes as the number of base learners increases. This effect exists in $U$-statistics as well, as the estimated variance decreases as $B_n$ increases. However for $U$-statistics, the variance estimates tend to underestimate when large kernels are used (Figure \ref{fig:ve_u_250} and \ref{fig:ve_u_400}). This is caused by the fact that the sampling scheme with $U$-statistics is not equivalent to sampling from the empirical distribution, especially when the kernel size is large. Both perspectives on variance estimation, either on estimating the variance of conditional expectation or resorting to Infinitesimal Jackknife, are based on the idea of using the empirical distribution of the data to approximate the true underlying distribution. $U$-statistics, which operate by sampling without replacement, are not equivalent to sampling from empirical distribution, thus resulting in the underestimation phenomenon. In \cite{MR3862353}, the authors use a correction factor $\frac{n(n-1)}{(n - k_n)^2}$ as an empirical adjustment for this effect. Figure \ref{fig:u_wager} in Appendix \ref{app:u_cor} shows the result when this correction is applied (denoted by corrected-IJ). Empirically the correction mitigates the underestimation bias, and exhibits a similar pattern as $V$-statistics: a bias due to overestimation of $\zeta_{1, k_n}$. 

\subsection{A Bias-corrected Estimator}\label{sec:bce}

In this section, we present a bias-corrected estimator for $\zeta_{1, k_n}$ under the framework of $V$-statistics. We use an ANOVA-like estimation of variance components similar to \cite{MR2844419}. Derivations are collected in Appendix \ref{app:bc_dev}.

Following notation form Section \ref{sec:BM}, define
$$
SS_{\tau} = \sum_{i=1}^{n}N_i\left(m_i - \bar{m}\right)^2
$$
and 
$$
SS_\epsilon = \sum_{i=1}^n\sum_{b=1}^{B_n} N_{i, b}\left(h_b - m_i\right)^2.
$$

A bias-corrected estimate is given by
$$
\hat{\zeta}_{1,k_n} = \frac{SS_{\tau} - \left(n-1\right)\hat{\sigma}_{\epsilon}^2}{C - \sum_{i=1}^nN_i^2/C}
$$
where $C = \sum_{i=1}^nN_i = B_n k_n$ and $\hat{\sigma}_{\epsilon}^2 = \frac{SS_{\epsilon}}{C - n}$.

As a special case for the \emph{Balanced Subsample Structure}, we have $N_1 = N_2 = \ldots = N_i = r_n$, then 
$$
\hat{\sigma}_{\epsilon}^2 = \frac{SS_{\epsilon}}{C - n} = \frac{1}{n\left(r_n - 1\right)}\sum_{i=1}^n\sum_{b=1}^{B_n} N_{i, b}\left(h_b - m_i\right)^2
$$
and 
\begin{equation}\label{eqn:v_cor}
    \hat{\zeta}_{1,k_n} = \frac{1}{n-1}\sum_{i=1}^{n}\left(m_i - \bar{m}\right)^2 - \frac{1}{r_n}\hat{\sigma}_{\epsilon}^2. 
\end{equation}

The calculation for $\hat{\zeta}_{1,k_n}$ in (\ref{eqn:v_cor}) may seem complicated at first. In Appendix \ref{app:simpler}, we show that under Balanced Subsample Structure $\hat{\zeta}_{1,k_n} \approx \frac{1}{n - 1}\sum_{i=1}^{n}(m_i - \bar{m})^2 - \frac{1}{B_n}\frac{n}{k_n}\hat{\zeta}_{k_n,k_n}^{\text{BM}}$, which is simply the original BM estimator minus a correction term calculated from $\hat{\zeta}_{k_n,k_n}^{\text{BM}}$. It indicates that one can actually calculate the bias-corrected estimator without any extra computation efforts. In this case, the estimate for the limiting variance $\frac{k_{n}^{2}}{n} \zeta_{1,k_n} + \frac{1}{B_n} \zeta_{k_n,k_n}$ is $\frac{k_n^2}{n(n - 1)}\sum_{i=1}^{n}(m_i - \bar{m})^2 - \frac{k_n - 1}{B_n}\hat{\zeta}_{k_n,k_n}^{\text{BM}}$, where it's clear that the incomplete part is negligible to the bias correction term we need to apply.


Figure \ref{fig:vi_cor} show the results for this bias-corrected estimator (we call it corrected-V in the remaining part of this paper) compared with BM and IJ under the framework of $V$-statistics. Here we no longer display IM since it's systematically worse. We can see that even with only 100 base learners, the bias-corrected estimator achieves relatively accurate estimation of the variance. The bias-corrected term may introduce some instability when $B_n$ is very small, but for a moderate size $B_n$ it has much lower bias compared to BM and IJ. 

\begin{figure}
    \centering
    \includegraphics[width=\textwidth]{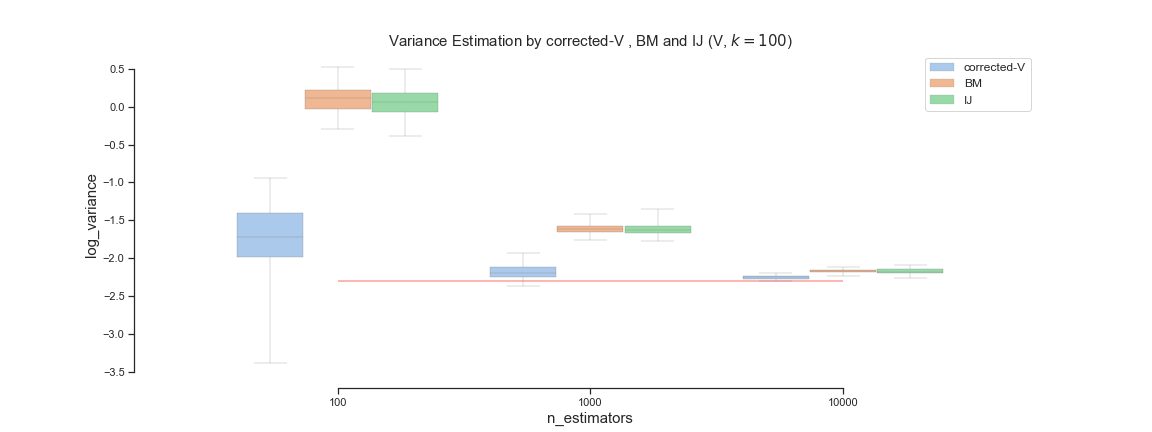}
    \caption{Variance estimation by three different methods: corrected-V, BM and IJ. The red line denotes true (log) variance obtained by generating data, training the ensemble 100 times and calculating the empirical variance of predictions.}
    \label{fig:vi_cor}
\end{figure}

It is worth pointing out that this bias correction method does not work for $U$-statistics, for the same reason mentioned before: the sampling schema is not equivalent to sampling from empirical distribution. Figure \ref{fig:ui_cor} in Appendix \ref{app:u_cor} shows this result under the framework of $U$-statistics. The blue box plot shows that the bias-corrected estimator over-corrects the variance. In \cite{MR3909963}, the authors developed a method called the bootstrap of little bags to estimate variance based on the work of \cite{MR2654590}. They also encountered the challenge of negative variance when $B_n$ is small. In their software, an improper uniform prior over $[0, \infty)$ was employed to help mitigate this issue. We conjecture that the phenomenon also stems from the mechanism of sampling without replacement. In Appendix \ref{app:u_cor}, we present an empirically accurate correction to $U$-statistics as well.

\section{Randomized Ensembles}\label{sec:re}

As briefly mentioned before, randomized ensembles are widely used in practice. A general principle to achieve good performance in ensembles is to make individual learners both accurate and diverse \cite{MR3184068}. To increase diversity, randomization is added to each base learner. For example in random forests \cite{breiman2001random}, each split is chosen from a randomly select subset of all possible features. 

Similar to \cite{peng2019asymptotic}, we define the notion of a \emph{generalized} complete $V$-statistic
\begin{equation}\label{eqn:crv}
    V_{n, k_n, \omega} = n^{-k_n}\sum_{i_1=1}^n\ldots\sum_{i_{k_n}=1}^nh_{k_n}\left(Z_{i_1}, \ldots, Z_{i_{k_n}}; \omega\right).
\end{equation}
Note that for each kernel $h_{k_n}$ we consider an i.i.d.\ sample of random $\omega_i$ but the subscript is dropped for notational convenience.

Similarly define the generalized incomplete statistic by
\begin{equation}\label{eqn:rv}
    V_{n, k_n, B_n, \omega} = \frac{1}{B_n}\sum_{i}h_{k_n}\left(Z_{i_1}, \ldots, Z_{i_{k_n}}; \omega\right).    
\end{equation}

Following the same idea developed in \cite{MR3491120} and \cite{MR3862353}, consider the expected version of (\ref{eqn:crv})
\begin{equation}\label{eqn:erv}
    V^*_{n, k_n, \omega} = \mathbb{E}_\omega V_{n, k_n, \omega} = n^{-k_n}\sum_{i_1=1}^n\ldots\sum_{i_{k_n}=1}^nE_{\omega}h_{k_n}\left(Z_{i_1}, \ldots, Z_{i_{k_n}}; \omega\right)
\end{equation}
where the expectation is taken over the randomization parameter $\omega$. In this case, $V^*_{n, k_n, \omega}$ can be viewed as a non-randomized $V$-statistic with kernel $h_{k_n}^E = \mathbb{E}_\omega h_{k_n}$ where Theorem \ref{thm:v2} applies. We state this result formally in the following corollary.

\begin{corollary}\label{cor:cv}
Let $Z_1, Z_2, \ldots, Z_n \stackrel{iid}{\sim} F_Z$ and let $V_{n, k_n, \omega}$ be a generalized complete $V$-statistic defined in (\ref{eqn:crv}) and the corresponding expected version $V^*_{n, k_n, \omega}$ in (\ref{eqn:erv}). Under the same conditions as Theorem \ref{thm:v2}, we have
$$
\frac{\left(V^*_{n, k_n, \omega} - \theta_{k_n}\right)}{\sqrt{\frac{k_n^2}{n}\zeta^*_{1, k_n} }}\stackrel{d}{\to} \mathcal{N}(0, 1),
$$
where all parameters $\theta_{k_n}, \zeta^*_{1, k_n}, \zeta^*_{k_n, k_n}$ are defined using new non-randomized kernel $h_{k_n}^E$ instead of $h_{k_n}$.
\end{corollary}

Given this, in order to retain the asymptotic normality of the corresponding randomized case (\ref{eqn:rv}), there are two steps: first we show that $\frac{V_{n, k_n, \omega} - V^*_{n, k_n, \omega}}{\text{Var}(V^*_{n, k_n, \omega})} \stackrel{P}{\to} 0$ and thus $V_{n, k_n, \omega}$ has the same asymptotic distribution as $V^*_{n, k_n, \omega}$. Then the asymptotics of $V_{n, k_n, B_n, \omega}$ can be derived from that of $V_{n, k_n, \omega}$. 

\begin{theorem}\label{thm:rv}
Let $V_{n, k_n, B_n, \omega}$ be a generalized incomplete $V$-statistic of the form defined in (\ref{eqn:rv}). Further assume the corresponding statistic $V^*_{n, k_n, \omega}$ in (\ref{eqn:erv}) satisfies Corollary \ref{cor:cv} and $\lim_{n \to \infty}k
_n^2 \zeta^*_{1, k_n} > 0$. Then as long as
$$
\lim_{n \to \infty}\mathbb{E} \left( h_{k_n} \left( Z_{i_1}, \ldots, Z_{i_{k_n}}; \omega \right) - \mathbb{E}_\omega h_{k_n}\left(Z_{i_1}, \ldots, Z_{i_{k_n}}; \omega \right) \right) \neq \infty,
$$
we have
$$
\frac{\left(V_{n, k_n, B_n, \omega} - \theta_{k_n}\right)}{\sqrt{\frac{k_n^2}{n}\zeta^*_{1, k_n} + \frac{1}{B_n}\zeta_{k_n, k_n}}}\stackrel{d}{\to} \mathcal{N}(0, 1).
$$
Here, $\zeta_{k_n,k_n} = \text{var} \left(h_{k_n}\left(Z_1, \ldots, Z_{k_n}, \omega\right) \right)$ is the variance across individual randomized kernels, and all parameters $\theta_{k_n}, \zeta^*_{1, k_n}, \zeta^*_{k_n, k_n}$ are defined using new kernel $h_{k_n}^E$ instead of $h_{k_n}$ as in Corollary \ref{cor:cv}.
\end{theorem}

The analysis for randomized ensembles in \cite{peng2019asymptotic} directly treats the randomized kernel, rather than first establishing a result for the expectation over $\omega$. This is easier in the framework of $U$-statistics; in the case of $V$-statistics, the new kernel $\phi_{k_n}^*$ constructed as in Section \ref{sec:vasu} no longer has independent randomization parameters, since different $\phi_{k_n}^*$ might share the same kernel $h_{k_n}$.

\section{Empirical Studies}\label{sec:es}

Here, we conduct two suites of experiments. All simulations are implemented in Python. For building random forests, we apply \textit{RandomForestClassifier} and \textit{RandomForestRegressor} from \textit{scikit-learn} \cite{MR1091842}. Unless otherwise noted, default parameter values are used. 

\subsection{Predictive Performance}

In this section, we evaluate the predictive performance for different sampling strategies. In particular, we focus on the scenario of $U$-statistics (sample without replacement) and $V$-statistics (sample with replacement), and varying subsample size (proportion of $0.2, 0.4, 0.6, 0.8, 1.0$ of the size of training data). We address the following two questions empirically:
\begin{enumerate}
    \item Should we subsample with or without replacement in terms of prediction performance?
    \item What is the best subsample size?
\end{enumerate}

There are six datasets taken from UCI Machine learning Repository (see Appendix \ref{app:datasets} for details) and we also include a regression function (denoted by MARS) which was initially considered by \cite{MR1091842} for multivariate adaptive regression splines, and has since been used as a benchmark in many random forests publications.

Each model is built using 100 trees and to full depth until a leaf is pure or contain fewer than 2 data points. $20\%$ of samples are left as test set. For classification, $\sqrt{p}$ of features are considered when searching for best splits, and $\floor{p / 3}$ for regression. Table \ref{table:pred} summarizes our results. The first three datasets are regression tasks for which we report root-mean-squared error and the remaining four are classification with accuracy given by correct classification percentage. We repeat the process 20 times and denote standard error in parenthesis. Top performance entries are marked in bold separately for sampling with vs. without replacement and an asterisk indicates the best performance across all scenarios.

We make two observations here. First, for both sampling with replacement or without, there is a best subsample size for prediction, though the proportion varies across different datasets. Accuracy decreases as the subsample size moves away from the ideal proportion. It is worth noting that performance discrepancy, which is defined to be the maximum performance difference across five subsample size settings, is generally larger in the case of $U$-statistics than $V$-statistics. For example, in diabetes dataset, there is a 4.1357 RMSE difference in $U$-statistics scenario compared to 1.03 in $V$-statistics. Similarly for retinopathy dataset, the accuracy discrepancy is $2.79\%$ versus $1.26\%$. It may suggest that sampling with replacement is more robust to changes in subsample sizes. 

On the other hand, we did not see an obvious performance gap between two sampling techniques. The best result can be generated either by sampling with replacement or without depending on the specific data at hand. In practice, one will need to use cross validation to choose the best sampling strategy and subsample size, if prediction performance is the top concern.

\begin{landscape}

\begin{table}[]
\caption{Predictive performance on seven datasets under different sampling strategies}
\centering
\begin{tabular}{|c|c|c|c|c|c|c|c|c|c|c|c|c|}
\hline
\multirow{2}{*}{Dataset} & \multirow{2}{*}{n} & \multirow{2}{*}{p} & \multicolumn{5}{c|}{$U$-statistics (sample without replacement)}                                                                                                                                                                                                                                                                                      & \multicolumn{5}{c|}{$V$-statistics (sample with replacement)}                                                                                                                                                                                                                                                            \\ \cline{4-13} 
                         &                    &                    & 0.2                                                        & 0.4                                                        & 0.6                                                        & 0.8                                                        & 1                                                                  & 0.2                                                        & 0.4                                                        & 0.6                                                        & 0.8                                                        & 1                                                                  \\ \hhline{|=|=|=|=|=|=|=|=|=|=|=|=|=|}
boston                   & 506                & 13                 & \begin{tabular}[c]{@{}c@{}}3.9696\\ (0.1112)\end{tabular}  & \begin{tabular}[c]{@{}c@{}}3.5754\\ (0. 0788)\end{tabular} & \begin{tabular}[c]{@{}c@{}}3.3225\\ (0.1131)\end{tabular}  & \begin{tabular}[c]{@{}c@{}}3.2103\\ (0.1050)\end{tabular}  & \begin{tabular}[c]{@{}c@{}}\textbf{3.0588*}\\ (0.0787)\end{tabular} & \begin{tabular}[c]{@{}c@{}}4.0358\\ (0.0656)\end{tabular}  & \begin{tabular}[c]{@{}c@{}}3.6799\\ (0.0758)\end{tabular}  & \begin{tabular}[c]{@{}c@{}}3.4705\\ (0.0936)\end{tabular}  & \begin{tabular}[c]{@{}c@{}}3.3738\\ (0.1057)\end{tabular}  & \begin{tabular}[c]{@{}c@{}}\textbf{3.2944}\\ (0.0887)\end{tabular} \\ \hline
diabetes                 & 442                & 10                 & \begin{tabular}[c]{@{}c@{}}\textbf{54.6137}\\ (0.6936)\end{tabular} & \begin{tabular}[c]{@{}c@{}}54.6289\\ (0.5164)\end{tabular} & \begin{tabular}[c]{@{}c@{}}55.5140\\ (0.6097)\end{tabular} & \begin{tabular}[c]{@{}c@{}}56.9017\\ (0.5681)\end{tabular} & \begin{tabular}[c]{@{}c@{}}58.7494\\ (0.8187)\end{tabular}         & \begin{tabular}[c]{@{}c@{}}\textbf{54.3830*}\\ (0.7100)\end{tabular} & \begin{tabular}[c]{@{}c@{}}54.4282\\ (0.6347)\end{tabular} & \begin{tabular}[c]{@{}c@{}}54.5220\\ (0.7739)\end{tabular} & \begin{tabular}[c]{@{}c@{}}55.0782\\ (0.5642)\end{tabular} & \begin{tabular}[c]{@{}c@{}}55.4130\\ (0.8877)\end{tabular}         \\ \hline
MARS                     & 500                & 5                  & \begin{tabular}[c]{@{}c@{}}3.2106\\ (0.0648)\end{tabular}  & \begin{tabular}[c]{@{}c@{}}2.8850\\ (0.0980)\end{tabular}  & \begin{tabular}[c]{@{}c@{}}2.7211\\ (0.0845)\end{tabular}  & \begin{tabular}[c]{@{}c@{}}2.6329\\ (0.0482)\end{tabular}  & \begin{tabular}[c]{@{}c@{}}\textbf{2.5465*}\\ (0.0484)\end{tabular}          & \begin{tabular}[c]{@{}c@{}}3.2908\\ (0.1059)\end{tabular}  & \begin{tabular}[c]{@{}c@{}}2.9706\\ (0.0644)\end{tabular}  & \begin{tabular}[c]{@{}c@{}}2.8201\\ (0.0683)\end{tabular}  & \begin{tabular}[c]{@{}c@{}}2.7498\\ (0.0591)\end{tabular}  & \begin{tabular}[c]{@{}c@{}}\textbf{2.6965}\\ (0.0425)\end{tabular}          \\ \hhline{|=|=|=|=|=|=|=|=|=|=|=|=|=|}
iris                     & 150                & 4                  & \begin{tabular}[c]{@{}c@{}}\textbf{0.9667*}\\ (0.0000)\end{tabular}  & \begin{tabular}[c]{@{}c@{}}\textbf{0.9667*}\\ (0.0000)\end{tabular}  & \begin{tabular}[c]{@{}c@{}}0.9633\\ (0.0100)\end{tabular}  & \begin{tabular}[c]{@{}c@{}}0.9500\\ (0.0167)\end{tabular}  & \begin{tabular}[c]{@{}c@{}}0.9333\\ (0.0000)\end{tabular}          & \begin{tabular}[c]{@{}c@{}}0.9650\\ (0.0073)\end{tabular}  & \begin{tabular}[c]{@{}c@{}}\textbf{0.9667*}\\ (0.0000)\end{tabular}  & \begin{tabular}[c]{@{}c@{}}0.9617\\ (0.0119)\end{tabular}  & \begin{tabular}[c]{@{}c@{}}\textbf{0.9667*}\\ (0.0000)\end{tabular}  & \begin{tabular}[c]{@{}c@{}}0.9633\\ (0.0100)\end{tabular}          \\ \hline
digits                   & 1797               & 64                 & \begin{tabular}[c]{@{}c@{}}0.9507\\ (0.0053)\end{tabular}  & \begin{tabular}[c]{@{}c@{}}0.9660\\ (0.0034)\end{tabular}  & \begin{tabular}[c]{@{}c@{}}0.9688\\ (0.0043)\end{tabular}  & \begin{tabular}[c]{@{}c@{}}\textbf{0.9735*}\\ (0.0048)\end{tabular}  & \begin{tabular}[c]{@{}c@{}}0.9731\\ (0.0045)\end{tabular}          & \begin{tabular}[c]{@{}c@{}}0.9481\\ (0.0054)\end{tabular}  & \begin{tabular}[c]{@{}c@{}}0.9588\\ (0.0058)\end{tabular}  & \begin{tabular}[c]{@{}c@{}}0.9651\\ (0.0049)\end{tabular}  & \begin{tabular}[c]{@{}c@{}}0.9681\\ (0.0048)\end{tabular}  & \begin{tabular}[c]{@{}c@{}}\textbf{0.9689}\\ (0.0036)\end{tabular}          \\ \hline
retinopathy                     & 1151                & 19                 & \begin{tabular}[c]{@{}c@{}}\textbf{0.6935}\\ (0.0156)\end{tabular}  & \begin{tabular}[c]{@{}c@{}}0.6825\\ (0.0112)\end{tabular}  & \begin{tabular}[c]{@{}c@{}}0.6766\\ (0.0114)\end{tabular}  & \begin{tabular}[c]{@{}c@{}}0.6708\\ (0.0082)\end{tabular}  & \begin{tabular}[c]{@{}c@{}}0.6656\\ (0.0125)\end{tabular}          & \begin{tabular}[c]{@{}c@{}}0.6892\\ (0.0135)\end{tabular}  & \begin{tabular}[c]{@{}c@{}}\textbf{0.6946*}\\ (0.0126)\end{tabular}  & \begin{tabular}[c]{@{}c@{}}0.6868\\ (0.0125)\end{tabular}  & \begin{tabular}[c]{@{}c@{}}0.6846\\ (0.0109)\end{tabular}  & \begin{tabular}[c]{@{}c@{}}0.6820\\ (0.0117)\end{tabular}          \\ \hline
breast\_cancer           & 569                & 30                 & \begin{tabular}[c]{@{}c@{}}\textbf{0.9772*}\\ (0.0070)\end{tabular}  & \begin{tabular}[c]{@{}c@{}}0.9746\\ (0.0038)\end{tabular}  & \begin{tabular}[c]{@{}c@{}}0.9750\\ (0.0064)\end{tabular}  & \begin{tabular}[c]{@{}c@{}}0.9741\\ (0.0019)\end{tabular}  & \begin{tabular}[c]{@{}c@{}}0.9745\\ (0.0026)\end{tabular}          & \begin{tabular}[c]{@{}c@{}}0.9741\\ (0.0071)\end{tabular}  & \begin{tabular}[c]{@{}c@{}}0.9737\\ (0.0039)\end{tabular}  & \begin{tabular}[c]{@{}c@{}}0.9746\\ (0.0038)\end{tabular}  & \begin{tabular}[c]{@{}c@{}}\textbf{0.9759}\\ (0.0047)\end{tabular}  & \begin{tabular}[c]{@{}c@{}}0.9754\\ (0.0035)\end{tabular}          \\ \hline
\end{tabular}
\end{table}
\label{table:pred}
\end{landscape}

\subsection{Asymptotic Normality and Variance Estimation}\label{sec:v_sim}

In this section, we illustrate empirically the asymptotic normality property and variance estimation algorithms for $V$-statistics. The underlying regression function is generated according to the MARS function  \cite{MR1091842, MR3491120}: $y = f(x) = 10\sin(\pi x_1x_2) + 20(x_3 - 0.05)^2 + 10x_4 + 5x_5 + \epsilon $, where $\mathcal{X} \sim U([0, 1]^5)$ and $\epsilon \sim \mathcal{N}(0, 1)$.

Our simulation runs for 500 iterations. In each iterations we generate $n = 500$ training observations and train random forests with subsample size $k = 100, 250, 500$ and the number of trees $B = 500, 1000, 2500, 5000$. We make evaluation on three test points: $p_1 = [0.5, 0.5, 0.5, 0.5, 0.5]$ and $p_2$, $p_3$ are randomly drawn from $U([0, 1]^5)$.

For each $k$ and $B$, let $\hat{f}_{i, j}$ denote the prediction at the $i$th data point and $j$th iteration ($i = 1,2,3$ and $j = 1,\ldots,500$). Similarly $\hat{V}_{u, ij}$ and $\hat{V}_{c, ij}$ are the variance estimates for IJ and corrected-V respectively.
The following three metrics are reported: 
\begin{itemize}
    \item Normality: We test the normality of predictions $\hat{f}_{i, j}$ for $j = 1, 2, \ldots 500$ based on \cite{MR323010} and \cite{MR339372} which combine skew and kurtosis, and is implemented by \textit{scipy.stats.normaltest}\footnote{\href{https://docs.scipy.org/doc/scipy/reference/generated/scipy.stats.normaltest.html}{https://docs.scipy.org/doc/scipy/reference/generated/scipy.stats.normaltest.html}} in Python. In Table \ref{table:v_sim}, we report test statistics and corresponding p-values (in parenthesis). For small subsample size $k = 100$ (Table \ref{table:k100}), all cases except one give p-values larger than 0.05, which indicates that the predictions do conform to normality. For larger subsample sizes $k = 250$ (Table \ref{table:k250}) and $k = 500$ (Table \ref{table:k500}), we can see that the normality property begins to break down for some cases. 
    \item Variance ratio: The estimated variance for each setting is given by $\bar{\hat{V}}_{u, i} = \frac{1}{500}\sum_{t = 1}^{500} \hat{V}_{u, ij}$ and $\bar{\hat{V}}_{c, i} = \frac{1}{500}\sum_{t = 1}^{500} \hat{V}_{c, ij}$. And true variance $V(\hat{f}_i)$ is approximated by the empirical variance of $\hat{f}_{i, j}$ for $j = 1, 2, \ldots 500$. The variance ratio is defined as $\frac{\bar{\hat{V}}_{u, i}}{V(\hat{f}_i)}$ and $\frac{\bar{\hat{V}}_{c, i}}{V(\hat{f}_i)}$, where a value close to 1 is ideal. We can see similar patterns across three tables. The original version of IJ produces highly biased variance estimates, where the bias diminishes as the number of trees $B$ becomes larger. The bias-corrected version successfully alleviates the issue. For $k = 100$, it starts to produce reasonable estimates for $1000$ trees, and the variance ratios are close to one for larger $B$ values. We can also see that it becomes harder to estimate the variance as the subsample sizes grow. 
    \item Coverage probability: constructing $95\%$ confidence intervals by $\hat{f}_{i, j} \pm 1.96 \times \sqrt{\hat{V}_{u, ij}}$ or $\hat{f}_{i, j} \pm 1.96 \times \sqrt{\hat{V}_{c, ij}}$ for $j = 1, 2, \ldots 500$ in each setting and we can calculate a coverage probability by checking whether the expected prediction value (approximated by $\bar{\hat{f}}_i = \frac{1}{500}\sum_{t=1}^{500}\hat{f}_{i, j}$) falls into this interval. This is strongly related to our results for variance ratios. A larger variance ratios will produce conservative intervals, thus generating higher coverage probability. The bias-corrected algorithm produces coverage probability close to 0.95 with reasonable number of base learners.
\end{itemize}

\begin{table}[]
     \caption{Asymptotic normality and variance estimation results for MARS function}
    \centering
    \begin{adjustbox}{scale=0.9}
    \begin{subtable}[h]{\textwidth}
    \centering
\begin{tabular}{|c|c|c|c|c|c|c|c|}
\hline
\multicolumn{2}{|c|}{\multirow{2}{*}{}} & \multicolumn{2}{c|}{$p_1$}              & \multicolumn{2}{c|}{$p_2$}              & \multicolumn{2}{c|}{$p_3$}              \\ \cline{3-8} 
\multicolumn{2}{|c|}{}                  & original          & corrected        & original          & corrected        & original          & corrected        \\ \hhline{|=|=|=|=|=|=|=|=|}
\multirow{3}{*}{B = 500}   & normality  & \multicolumn{2}{c|}{3.5285 (0.1713)} & \multicolumn{2}{c|}{0.2274 (0.8925)} & \multicolumn{2}{c|}{0.0762 (0.9626)} \\ \cline{2-8} 
                           & var\_ratio & 7.6985            & 1.4962           & 8.3871            & 1.5648           & 7.3020            & 1.4229           \\ \cline{2-8} 
                           & coverage   & 100.0             & 97.2             & 100.0             & 96.4             & 100.0             & 96.6             \\ \hline
\multirow{3}{*}{B = 1000}  & normality  & \multicolumn{2}{c|}{1.2341 (0.5395)} & \multicolumn{2}{c|}{1.1715 (0.5567)} & \multicolumn{2}{c|}{8.0994 (0.0174)} \\ \cline{2-8} 
                           & var\_ratio & 4.9540            & 1.3839           & 4.9290            & 1.3116           & 4.4433            & 1.2472           \\ \cline{2-8} 
                           & coverage   & 100.0             & 96.4             & 100.0             & 96.6             & 100.0             & 96.4             \\ \hline
\multirow{3}{*}{B = 2500}  & normality  & \multicolumn{2}{c|}{1.9708 (0.3733)} & \multicolumn{2}{c|}{2.8710 (0.2380)} & \multicolumn{2}{c|}{1.8699 (0.3926)} \\ \cline{2-8} 
                           & var\_ratio & 2.3068            & 1.0328           & 2.5635            & 1.1058           & 2.2736            & 1.0348           \\ \cline{2-8} 
                           & coverage   & 99.2              & 94.2             & 99.8              & 94.6             & 99.4              & 94.4             \\ \hline
\multirow{3}{*}{B = 5000}  & normality  & \multicolumn{2}{c|}{3.2266 (0.1992)} & \multicolumn{2}{c|}{0.8750 (0.6456)} & \multicolumn{2}{c|}{0.5155 (0.7728)} \\ \cline{2-8} 
                           & var\_ratio & 1.7073            & 1.0305           & 1.7748            & 1.0460           & 1.6432            & 0.9884           \\ \cline{2-8} 
                           & coverage   & 98.0              & 95.0             & 99.2              & 94.8             & 99.0              & 95.0             \\ \hline
\end{tabular}
       \caption{k = 100}
       \label{table:k100}
    \end{subtable}
    \end{adjustbox}
    
    \hfill
    
    \begin{adjustbox}{scale=0.9}
    \begin{subtable}[h]{\textwidth}
        \centering
\begin{tabular}{|c|c|c|c|c|c|c|c|}
\hline
\multicolumn{2}{|c|}{\multirow{2}{*}{}} & \multicolumn{2}{c|}{$p_1$}               & \multicolumn{2}{c|}{$p_2$}              & \multicolumn{2}{c|}{$p_3$}                   \\ \cline{3-8} 
\multicolumn{2}{|c|}{}                  & original          & corrected         & original          & corrected        & original             & corrected          \\ \hhline{|=|=|=|=|=|=|=|=|}
\multirow{3}{*}{B = 500}   & normality  & \multicolumn{2}{c|}{10.1465 (0.0063)} & \multicolumn{2}{c|}{0.8482 (0.6544)} & \multicolumn{2}{c|}{2.0166 (0.3648)}      \\ \cline{2-8} 
                          & var\_ratio & 9.4265            & 2.7234            & 9.9809            & 2.8231           & 9.5969               & 2.7436             \\ \cline{2-8} 
                          & coverage   & 100.0             & 100.0             & 100.0             & 99.4             & 100.0                & 99.4               \\ \hline
\multirow{3}{*}{B = 1000}  & normality  & \multicolumn{2}{c|}{7.0852 (0.0289)}  & \multicolumn{2}{c|}{2.5277 (0.2826)} & \multicolumn{2}{c|}{25.8658 (2.4171e-06)} \\ \cline{2-8} 
                          & var\_ratio & 5,7448            & 2.0294            & 5.6789            & 1.9683           & 5.1754               & 1.8041             \\ \cline{2-8} 
                          & coverage   & 100.0             & 98.0              & 100.0             & 99.0             & 100.0                & 98.4               \\ \hline
\multirow{3}{*}{B = 2500}  & normality  & \multicolumn{2}{c|}{4.8638 (0.0878)}  & \multicolumn{2}{c|}{6.6417 (0.0361)} & \multicolumn{2}{c|}{7.0543 (0.0294)}      \\ \cline{2-8} 
                          & var\_ratio & 2.5655            & 1.2496            & 2.7176            & 1.2880           & 2.5462               & 1.2501             \\ \cline{2-8} 
                          & coverage   & 99.8              & 95.4              & 99.6              & 96.6             & 99.6                 & 95.6               \\ \hline
\multirow{3}{*}{B = 5000}  & normality  & \multicolumn{2}{c|}{5.9457 (0.0512)}  & \multicolumn{2}{c|}{2.1790 (0.3364)} & \multicolumn{2}{c|}{5.8061 (0.0549)}      \\ \cline{2-8} 
                          & var\_ratio & 1.7179            & 1.0630            & 1.9767            & 1.2059           & 1.8045               & 1.1108             \\ \cline{2-8} 
                          & coverage   & 97.4              & 95.0              & 98.6              & 95.2             & 98.6                 & 95.2               \\ \hline
\end{tabular}
        \caption{k = 250}
        \label{table:k250}
     \end{subtable}
     \end{adjustbox}
     \hfill
     
    \begin{adjustbox}{scale=0.9}
    \begin{subtable}[h]{\textwidth}
    \centering
\begin{tabular}{|c|c|c|c|c|c|c|c|}
\hline
\multicolumn{2}{|c|}{\multirow{2}{*}{}} & \multicolumn{2}{c|}{$p_1$}              & \multicolumn{2}{c|}{$p_2$}               & \multicolumn{2}{c|}{$p_3$}                   \\ \cline{3-8} 
\multicolumn{2}{|c|}{}                  & original          & corrected        & original           & corrected        & original             & corrected          \\ \hhline{|=|=|=|=|=|=|=|=|}
\multirow{3}{*}{B = 500}   & normality  & \multicolumn{2}{c|}{4.5835 (0.1011)} & \multicolumn{2}{c|}{2.0985 (0.3502)}  & \multicolumn{2}{c|}{7.5323 (0.0231)}      \\ \cline{2-8} 
                           & var\_ratio & 10.7237           & 4.5438           & 11.6708            & 4.9121           & 11.5562              & 4.8322             \\ \cline{2-8} 
                           & coverage   & 100.0             & 100.0            & 100.0              & 100.0            & 100.0                & 99.8               \\ \hline
\multirow{3}{*}{B = 1000}  & normality  & \multicolumn{2}{c|}{7.2961 (0.0260)} & \multicolumn{2}{c|}{10.9307 (0.0042)} & \multicolumn{2}{c|}{26.6668 (1.6195e-06)} \\ \cline{2-8} 
                           & var\_ratio & 6.4412            & 3.0251           & 5.9518             & 2.7838           & 5.7957               & 2.7163             \\ \cline{2-8} 
                           & coverage   & 100.0             & 99.0             & 100.0              & 99.6             & 100.0                & 99.4               \\ \hline
\multirow{3}{*}{B = 2500}  & normality  & \multicolumn{2}{c|}{4.6806 (0.0963)} & \multicolumn{2}{c|}{4.4154 (0.1100)}  & \multicolumn{2}{c|}{6.7212 (0.0347)}      \\ \cline{2-8} 
                           & var\_ratio & 2.9094            & 1.6640           & 2.9235             & 1.6593           & 2.7023               & 1.5515             \\ \cline{2-8} 
                           & coverage   & 99.8              & 97.4             & 99.4               & 97.8             & 99.0                 & 95.8               \\ \hline
\multirow{3}{*}{B = 5000}  & normality  & \multicolumn{2}{c|}{5.4816 (0.0645)} & \multicolumn{2}{c|}{2.0471 (0.3593)}  & \multicolumn{2}{c|}{13.4452 (0.0012)}     \\ \cline{2-8} 
                           & var\_ratio & 1.7294            & 1.1652           & 2.1471             & 1.4412           & 1.8567               & 1.2449             \\ \cline{2-8} 
                           & coverage   & 97.2              & 95.0             & 98.4               & 95.4             & 98.6                 & 94.6               \\ \hline
\end{tabular}
\caption{k = 500}
\label{table:k500}
    \end{subtable}
    \end{adjustbox}
    
     \label{table:v_sim}

\end{table}

\section{Conclusion}\label{sec:con}

In this paper, we present a framework for analyzing the asymptotics of $V$-statistics where the kernal size $k_n$ grows with the number of samples $n$. It is shown that a central limit theorem can be established similar to the work in \cite{MR3491120}, \cite{MR3862353} and \cite{peng2019asymptotic}, which focus on the case of $U$-statistics. The result brings new insight into the analysis of ensemble methods. 

We also provide unified treatment of variance estimation in both $U$- and $V$-statistics. We observe that existing methods for estimating limiting variance exhibit severe bias and would require a prohibitively large number of base learners to achieve accurate results, hindering any practical applications such as constructing confidence intervals or conducting hypothesis tests. To this end, we propose a new method called Balanced Variance Estimation Method (BM), and carefully analyze its connection to other methods. In particular, we demonstrate an equivalence between BM and Infinitesimal Jackknife. Additionally, a bias correction method is developed which is shown to produce more accurate variance estimation with a moderate size of base learners. 

Practically, we would suggest sampling with replacement in building ensemble methods since the bias correction for $V$-statistics is theoretically sound and much less involved. How to theoretically analyze bias correction for $U$-statistics remains a promising future endeavor.

The condition $\lim_{n \to \infty}k_n^2 \zeta_{1, k_n} > 0$ required in both Theorem \ref{thm:v} and \ref{thm:rv} does not appear in previous literature. However, we believe it is generally satisfied with many base learners including trees, see \cite{peng2019asymptotic} for an in-depth analysis for the behavior of $\zeta_{1, k_n}$. We leave it as a future work to study the scenario when $\lim_{n \to \infty}k_n^2 \zeta_{1, k_n} \to 0$.

From another theoretical point of view, the analysis we provide here is essentially a reduction to $U$-statistics. We will further explore whether other approaches like Taylor expansion using differential methods \cite{MR595165} could be applied to attain similar results. 

\begin{appendix}

\section{Proofs}


\subsection{Proof of Theorem \ref{thm:vandu}}\label{p:thm:vandu}

\begin{proof}
By Slutsky's theorem, we only need to show $\frac{(V_{n, k_n} - U_{n, k_n})}{\sqrt{\text{Var}(U_{n, k_n})}} \overset{p}{\to} 0$. Since we assume $\lim_{n \to \infty} \text{Var}(\sqrt{n}U_{n, k_n}) > 0$, it suffices to prove $\sqrt{n}(V_{n, k_n} - U_{n, k_n}) \overset{p}{\to} 0$. We seek to prove $L^1$ convergence, which implies convergence in probability. According to Lemma \ref{lm:vtou}, $V_{n, k_n}$ could be written as
\begin{align*}
    V_{n, k_n} &= \sum_{j=1}^{k_n}\frac{j!S_{k_n}^{(j)}{n \choose j}}{n^{k_n}}U_{n}^{(j)} \\
    &= \frac{k_n!S_{k_n}^{(k_n)}{n \choose k_n}}{n^{k_n}}U_{n, k_n} + \sum_{j=1}^{k_n-1}\frac{j!S_{k_n}^{(j)}{n \choose j}}{n^{k_n}}U_{n}^{(j)} \\
    &= \frac{n(n-1)\ldots(n-k_n+1)}{n^{k_n}}U_{n, k_n} + \sum_{j=1}^{k_n-1}\frac{j!S_{k_n}^{(j)}{n \choose j}}{n^{k_n}}U_{n}^{(j)}.
\end{align*}

Since we assume the second moment of kernel $h$ is bounded, both $U_{n}^{(j)}$ and $V_{n}^{(j)}$ could also be bounded by a constant $C < \infty$. We have

\begin{align*}
    \mathbb{E}\left|\sqrt{n}\left(V_{n, k_n} - U_{n, k_n}\right)\right| &= \sqrt{n}\mathbb{E}\left|\left(\frac{n(n-1)\ldots(n-k_n+1)}{n^{k_n}} - 1\right)U_{n, k_n} + \sum_{j=1}^{k_n-1}\frac{j!S_{k_n}^{(j)}{n \choose j}}{n^{k_n}}U_{n}^{(j)}\right| \\
    & \leq \sqrt{n}\mathbb{E}\left|\left(\frac{n(n-1)\ldots(n-k_n+1)}{n^{k_n}} - 1\right)U_{n, k_n}\right| + \mathbb{E}\left|\sum_{j=1}^{k_n-1}\frac{j!S_{k_n}^{(j)}{n \choose j}}{n^{k_n}}U_{n}^{(j)}\right| \\
    & \leq C\left(\left|\sqrt{n}\left(\frac{n(n-1)\ldots(n-k_n+1)}{n^{k_n}} - 1\right)\right| + \sqrt{n}\sum_{j=1}^{k_n-1}\frac{j!S_{k_n}^{(j)}{n \choose j}}{n^{k_n}}\right).
\end{align*}

First it's easy to see

$$
\sqrt{n}\left(\frac{n(n-1)\ldots(n-k_n+1)}{n^{k_n}} - 1\right) \to 0
$$
as $n \to \infty$ and $k_n = o(n^{\frac{1}{4}})$.
An upper bound for $S_{k_n}^{(j)}$ is provided in \cite{MR241310}
$$
S_{k_n}^{(j)} \leq \frac{1}{2}{k_n \choose j}j^{k_n-j}.
$$

Thus, 

\begin{align*}
    \frac{j!S_{k_n}^{(j)}{n \choose j}}{n^{k_n}} &=  \frac{j!{n \choose j}}{n^j}\frac{S_{k_n}^{(j)}}{n^{k_n-j}} \\
    &\leq \frac{1}{2}\frac{j!{n \choose j}}{n^j} \frac{{k_n \choose j}j^{k_n-j}}{n^{k_n-j}} \\
    &\leq \frac{1}{2}\frac{j!{n \choose j}}{n^j} \frac{k_n^{k_n-j}k_n^{k_n-j}}{n^{k_n-j}} \\
    &= \frac{1}{2}\frac{j!{n \choose j}}{n^j} \left(\frac{k_n^2}{n}\right)^{k_n-j}\\
    &\leq\left(\frac{k_n^2}{n}\right)^{k_n-j}.
\end{align*}

Let $a_n = \frac{k_n^2}{\sqrt{n}}$, and we know $a_n \to 0$. Taking the sum yields

\begin{align*}
    \sqrt{n}\sum_{j=1}^{k_n-1}\frac{j!S_{k_n}^{(j)}{n \choose j}}{n^{k_n}} &\leq \sqrt{n}\sum_{j=1}^{k_n-1}\left(\frac{k_n^2}{n}\right)^{k_n-j} \\
    & \leq \sum_{j=1}^{k_n-1}\left(\frac{k_n^2}{\sqrt{n}}\right)^{k_n-j} \\
    &= \sum_{j=1}^{k_n-1}a_n^{k_n-j} \\
    &\leq \frac{a_n}{1 - a_n} \\
    &\to 0.
\end{align*}
We could conclude that $\mathbb{E}|\sqrt{n}(V_{n, k_n} - U_{n, k_n})| \to 0$.
\end{proof}

\subsection{Proof of Theorem \ref{thm:v}}\label{p:thm:v}

Since $k_n = o(n^{\frac{1}{4}})$ and $\lim_{n \to \infty}k_n^2 \zeta_{1, k_n} > 0$, the complete case follows directly from Theorem \ref{thm:vandu} and Theorem 1 in \cite{peng2019asymptotic}. we will need the following lemma for the incomplete case.

\begin{lemma}\label{lm:mul}
Let $a_1, a_2, \ldots$ be a sequence of constants such that $\lim_{n \to \infty}\frac{1}{n}\sum_{i=1}^na_i = 0$ and $\lim_{n \to \infty}\frac{1}{n}\sum_{i=1}^na_i^2 = \sigma^2$ and let random variables $M_1, \ldots, M_n$ have a multinomial distribution, multinomial$(B_n; \frac{1}{n}, \ldots, \frac{1}{n})$. Then as $B_n, n \to \infty$, the limiting distribution of 
$$
B_n^{-\frac{1}{2}}\sum_{i=1}^na_i\left(M_i-\frac{B_n}{n}\right)
$$
is $\mathcal{N}(0, \sigma^2)$.
\end{lemma}

\begin{proof}
The characteristic function of $(M_1, \ldots, M_n)$ is $(\frac{1}{n}e^{it_1} + \ldots + \frac{1}{n}e^{it_n})^{B_n}$ since it's multinomial$(B_n; \frac{1}{n}, \ldots, \frac{1}{n})$. Thus the characteristic function of $B_n^{-\frac{1}{2}}\sum_{i=1}^na_i\left(M_i-\frac{B_n}{n}\right)$ is given by
\begin{align*}
   \mathbb{E}\left(e^{itB_n^{-\frac{1}{2}}\sum_{i=1}^na_i\left(M_i-\frac{B_n}{n}\right)}\right) &= e^{-itB_n^{-\frac{1}{2}}\frac{B_n}{n}\sum_{i=1}^na_i} \mathbb{E}\left(e^{itB_n^{-\frac{1}{2}}\sum_{i=1}^na_iM_i}\right)\\
    &= e^{-it\bar{a}_nB_n^{\frac{1}{2}}}\left(\frac{1}{n}e^{ita_1B_n^{-\frac{1}{2}}} + \ldots + \frac{1}{n}e^{ita_nB_n^{-\frac{1}{2}}}\right)^{B_n} \\
    &= e^{-it\bar{a}_nB_n^{\frac{1}{2}}}\left(\frac{1}{n}\left(n + itB_n^{-\frac{1}{2}}\sum_{i=1}^na_i + \frac{1}{2}\left(itB_n^{-\frac{1}{2}}\right)^2\sum_{i=1}^na_i^2 + \ldots\right)\right)^{B_n} \\
    &= e^{-it\bar{a}_nB_n^{\frac{1}{2}}}\left(1 + itB_n^{-\frac{1}{2}}\bar{a}_n + \frac{1}{2}\sigma_n^2\left(itB_n^{-\frac{1}{2}}\right)^2 + o\left(B_n^{-1}\right)\right)^{B_n}
\end{align*}
where $\bar{a}_n = \frac{1}{n}\sum_{i=1}^na_i$ and $\sigma_n^2 = \frac{1}{n}\sum_{i=1}^na_i^2$.
Taking the logarithm gives
\begin{align*}
    \log \mathbb{E}\left(e^{itB_n^{-\frac{1}{2}}\sum_{i=1}^na_i\left(M_i-\frac{B_n}{n}\right)}\right) &= -it\bar{a}_nB_n^{\frac{1}{2}} + {B_n}\log\left(1 + itB_n^{-\frac{1}{2}}\bar{a}_n + \frac{1}{2}\sigma_n^2\left(itB_n^{-\frac{1}{2}}\right)^2 + o\left(B_n^{-1}\right)\right) \\
    &= -it\bar{a}_nB_n^{\frac{1}{2}} + {B_n}\left(itB_n^{-\frac{1}{2}}\bar{a}_n + \frac{1}{2}\left(\sigma_n^2 - \bar{a}_n^2\right)\left(itB_n^{-\frac{1}{2}}\right)^2\right) + o\left(1\right) \\
    &= -\frac{1}{2}\left(\sigma_n^2 - \bar{a}_n^2\right)t^2 + o\left(1\right).
\end{align*}
Since we assume tht $\bar{a_n} \to 0$ and $\sigma_n^2 \to \sigma^2$, the above quantity converges to $-\frac{1}{2}\sigma^2t^2$, which is the logarithm of the characteristic function of $\mathcal{N}(0, \sigma^2)$.
\end{proof}

Now we could prove the major part of Theorem \ref{thm:v}.

\begin{proof} 
Without loss of generality we will assume $\theta_{k_n} = 0$. Suppose $(M_1, \ldots, M_{n^{k_n}})$ have a multinomial distribution, multinomial $(B_n; \frac{1}{n^{k_n}}, \ldots, \frac{1}{n^{k_n}})$. We could rewrite $V_{n, k_n, B_n} $ as

\begin{align*}\label{eqn:v_decomp}
V_{n, k_n, B_n} &= \frac{1}{B_n}\sum_{i}h_{k_n}\left(Z_{i_1}, \ldots, Z_{i_{k_n}}\right) \\
&= \frac{1}{B_n}\sum_{i = 1}^{n^{k_n}}M_ih_{k_n}\left(Z_{i_1}, \ldots, Z_{i_{k_n}}\right) \\
&= \frac{1}{B_n}\sum_{i = 1}^{n^{k_n}}\left(M_i - \frac{B_n}{n^{k_n}} + \frac{B_n}{n^{k_n}}\right) h_{k_n}\left(Z_{i_1}, \ldots, Z_{i_{k_n}}\right) \\
&= \frac{1}{B_n}\left(M_i - \frac{B_n}{n^{k_n}}\right)h_{k_n}\left(Z_{i_1}, \ldots, Z_{i_{k_n}}\right) + \frac{1}{n^{k_n}}\sum_{i = 1}^{n^{k_n}}h_{k_n}\left(Z_{i_1}, \ldots, Z_{i_{k_n}}\right) \\
&= \frac{1}{B_n}\left(M_i - \frac{B_n}{n^{k_n}}\right)h_{k_n}\left(Z_{i_1}, \ldots, Z_{i_{k_n}}\right) + V_{n, k_n}.
\end{align*}

To show $\frac{V_{n, k_n, B_n}}{\sqrt{\frac{k_n^2}{n}\zeta_{1, k_n} + \frac{1}{B_n}\zeta_{k_n, k_n}}}\stackrel{d}{\to} \mathcal{N}(0, 1)$, it is equivalent to prove
$$
\lim_{n\to\infty} \mathbb{E}\left[\exp\left(it\frac{V_{n, k_n, B_n}}{\sqrt{\frac{k_n^2}{n}\zeta_{1, k_n} + \frac{1}{B_n}\zeta_{k_n, k_n}}}\right)\right] = \exp{\left(-\frac{1}{2}t^2\right)}.
$$

From the above decomposition of $V_{n, k_n, B_n}$, we have 
\begin{align*}
&\lim_{n\to\infty} \mathbb{E}\left[\exp\left(it\frac{V_{n, k_n, B_n}}{\sqrt{\frac{k_n^2}{n}\zeta_{1, k_n} + \frac{1}{B_n}\zeta_{k_n, k_n}}}\right)\right]\\
=& \lim_{n\to\infty} \mathbb{E}\left[\exp\left(it\frac{(\frac{1}{B_n}(M_i - \frac{B_n}{n^{k_n}})h_{k_n}(Z_{i_1}, \ldots, Z_{i_{k_n}}) + V_{n, k_n})}{\sqrt{\frac{k_n^2}{n}\zeta_{1, k_n} + \frac{1}{B_n}\zeta_{k_n, k_n}}}\right)\right] \\
=&\lim_{n\to\infty} \mathbb{E}\left[\mathbb{E}\left[\exp\left(it\frac{(\frac{1}{B_n}(M_i - \frac{B_n}{n^{k_n}})h_{k_n}(Z_{i_1}, \ldots, Z_{i_{k_n}}) + V_{n, k_n})}{\sqrt{\frac{k_n^2}{n}\zeta_{1, k_n} + \frac{1}{B_n}\zeta_{k_n, k_n}}}\right)|Z_1, \ldots, Z_n\right]\right] \\
=& \lim_{n\to\infty} \mathbb{E}\left[\exp\left(it\frac{V_{n, k_n}}{\sqrt{\frac{k_n^2}{n}\zeta_{1, k_n} + \frac{1}{B_n}\zeta_{k_n, k_n}}}\right)\mathbb{E}\left[\exp\left(it\frac{\frac{1}{B_n}(M_i - \frac{B_n}{n^{k_n}})h_{k_n}(Z_{i_1}, \ldots, Z_{i_{k_n}})}{\sqrt{\frac{k_n^2}{n}\zeta_{1, k_n} + \frac{1}{B_n}\zeta_{k_n, k_n}}}\right)|Z_1, \ldots, Z_n\right]\right] \\
\end{align*}

Since $\frac{V_{n, k_n}}{\sqrt{\frac{k_n^2}{n}\zeta_{1, k_n}}}\stackrel{d}{\to} \mathcal{N}(0, 1)$ any by Lemma \ref{lm:mul}, 
\begin{align*}
&\lim_{n\to\infty} \mathbb{E}\left[\exp\left(it\frac{V_{n, k_n, B_n}}{\sqrt{\frac{k_n^2}{n}\zeta_{1, k_n} + \frac{1}{B_n}\zeta_{k_n, k_n}}}\right)\right]\\
=& \lim_{n\to\infty} \mathbb{E}\left[\exp\left(it\frac{V_{n, k_n}}{\sqrt{\frac{k_n^2}{n}\zeta_{1, k_n} + \frac{1}{B_n}\zeta_{k_n, k_n}}}\right)\mathbb{E}\left[\exp\left(it\frac{\frac{1}{B_n}(M_i - \frac{B_n}{n^{k_n}})h_{k_n}(Z_{i_1}, \ldots, Z_{i_{k_n}})}{\sqrt{\frac{k_n^2}{n}\zeta_{1, k_n} + \frac{1}{B_n}\zeta_{k_n, k_n}}}\right)|Z_1, \ldots, Z_n\right]\right] \\
=& \lim_{n\to\infty} \mathbb{E}\left[\exp\left(it\frac{V_{n, k_n}}{\sqrt{\frac{k_n^2}{n}\zeta_{1, k_n} + \frac{1}{B_n}\zeta_{k_n, k_n}}}\right)\exp\left(-\frac{\frac{1}{B_n}\zeta_{k_n, k_n}}{2(\frac{k_n^2}{n}\zeta_{1, k_n} + \frac{1}{B_n}\zeta_{k_n, k_n})}t^2\right)\right] \\
=& \exp\left(-\frac{\frac{k_n^2}{n}\zeta_{1, k_n}}{2\left(\frac{k_n^2}{n}\zeta_{1, k_n} + \frac{1}{B_n}\zeta_{k_n, k_n}\right)}t^2\right)\exp\left(-\frac{\frac{1}{B_n}\zeta_{k_n, k_n}}{2\left(\frac{k_n^2}{n}\zeta_{1, k_n} + \frac{1}{B_n}\zeta_{k_n, k_n}\right)}t^2\right) \\
=& \exp{\left(-\frac{1}{2}t^2\right)}.
\end{align*}


    
\end{proof}


\subsection{Proof of Theorem \ref{thm:bm2ij}}

\begin{proof}

In the case of \emph{Balanced Subsample Structure} where $r_n = \frac{B_n \times k_n}{n}$, we have $\bar{m} = \bar{h}$ and $N_i = r_n$ for all $i$.

First we can rewrite $\hat{\zeta}_{1,k_n}^{\text{BM}}$ as
\begin{align*}
    \hat{\zeta}_{1,k_n}^{\text{BM}} &= \frac{1}{n - 1}\sum_{i=1}^{n}\left(m_i - \bar{m}\right)^2 \\
    &= \frac{1}{n - 1}\sum_{i=1}^{n}\left(\sum_{b=1}^{B_n}\omega_{i, b}h_b - \bar{h}\right)^2 \\
    &= \frac{1}{n - 1}\sum_{i=1}^{n}\left(\sum_{b=1}^{B_n}\frac{N_{i, b}}{N_i}h_b - \bar{h}\right)^2 \\
    &= \frac{1}{n - 1}\frac{1}{r_n^2}\sum_{i=1}^{n}\left(\sum_{b=1}^{B_n}N_{i, b}\left(h_b - \bar{h}\right)\right)^2.
\end{align*}

Then for $\hat{V}_{\text{IJ}} = \sum_{i=1}^n \text{cov}^2(N_{i, b}, h_b)$, we look at each individual term
\begin{align*}
    \text{cov}\left(N_{i, b}, h_b\right) &= \frac{\sum_{b=1}^{B_n}(N_{i, b} - \bar{N}_i)(h_b - \bar{h})}{B_n} \\
    &= \frac{1}{B_n}\left(\sum_{b, x_i \in b}\left(N_{i, b} - \bar{N}_i\right)\left(h_b - \bar{b}\right) +  \sum_{b, x_i \notin b}\left(N_{i, b} - \bar{N}_i\right)\left(h_b - \bar{b}\right)\right) \\
    &= \frac{1}{B_n}\left(\sum_{b, x_i \in b}\left(N_{i, b} - \frac{k_n}{n}\right)\left(h_b - \bar{h}\right) +  \sum_{b, x_i \notin b}\left(0 - \frac{k_n}{n}\right)\left(h_b - \bar{h}\right)\right) \\
    & = \frac{1}{B_n}\left(\sum_{b, x_i \in b}N_{i, b}\left(h_b - \bar{h}\right) - \frac{k_n}{n} \sum_b\left(h_b - \bar{h}\right)\right) \\
    &= \frac{1}{B_n}\sum_{b, x_i \in b}N_{i, b}\left(h_b - \bar{h}\right) \\
    & = \frac{1}{B_n}\sum_{b = 1}^{B_n}N_{i, b}\left(h_b - \bar{h}\right)
\end{align*}

Combining two previous identities
\begin{align*}
    \hat{V}_{\text{IJ}} &= \sum_{i=1}^n \text{cov}^2(N_{i, b}, h_b) \\
    &=\frac{1}{B_n^2}\sum_{i=1}^n\left(\sum_{b = 1}^{B_n}N_{i, b}(h_b - \bar{b})\right)^2\\
    &= \frac{(n - 1)r_n^2}{B_n^2}\hat{\zeta}_{1,k_n}^{\text{BM}} \\
    &= (n - 1) \frac{r_n^2}{B_n^2 k_n^2}k_n^2\hat{\zeta}_{1,k_n}^{\text{BM}} \\
    &= (n-1)\frac{k_n^2}{n^2}\hat{\zeta}_{1,k_n}^{\text{BM}} \\
    &= \frac{n-1}{n}\frac{k_n^2}{n}\hat{\zeta}_{1,k_n}^{\text{BM}}
\end{align*}
as claimed.

\end{proof}

\subsection{Proof of Theorem \ref{thm:rv}}\label{p:thm:rv}

\begin{proof}
The assumption $\lim_{n \to \infty}k
_n^2 \zeta^*_{1, k_n} > 0$ implies $\lim_{n \to \infty}\text{Var}(\sqrt{n}V^*_{n, k_n, B_n,\omega}) > 0$. We first show the complete case. Similar to the proof of Theorem 2 in \cite{MR3491120}, we have
$$
\mathbb{E}(V_{n, k_n, \omega} - V^*_{n, k_n, \omega})^2 = \frac{1}{n^{k_n}}\mathbb{E}\left(h_{k_n}\left(Z_{i_1}, \ldots, Z_{i_{k_n}}; \omega) - \mathbb{E}_\omega h_{k_n}(Z_{i_1}, \ldots, Z_{i_{k_n}}; \omega\right)\right)^2.
$$
Thus,
\begin{align*}
    & \lim_{n \to \infty}\mathbb{E}\left(\frac{\sqrt{n}(V_{n, k_n, \omega} - V^*_{n, k_n, \omega})}{\sqrt{\text{Var}(\sqrt{n}V^*_{n, k_n, \omega})}}\right)^2 \\
    =& \lim_{n \to \infty}\mathbb{E} \frac{n}{n^{k_n}}\frac{1}{\text{Var}(\sqrt{n}V^*_{n, k_n, \omega})}\mathbb{E}\left(h_{k_n}\left(Z_{i_1}, \ldots, Z_{i_{k_n}}; \omega) - \mathbb{E}_\omega h_{k_n}(Z_{i_1}, \ldots, Z_{i_{k_n}}; \omega\right)\right)^2 \\
    =& 0
\end{align*}
since $\lim_{n \to \infty}\mathbb{E}(h_{k_n}(Z_{i_1}, \ldots, Z_{i_{k_n}}; \omega) - \mathbb{E}_\omega h_{k_n}(Z_{i_1}, \ldots, Z_{i_{k_n}}; \omega)) \neq \infty$, $\lim_{n \to \infty} \frac{n}{n^{k_n}} \to 0$ and $\lim_{n \to \infty}\text{Var}(\sqrt{n}V^*_{n, k_n, \omega}) > 0$.

The incomplete case follows exactly as in \cite{MR3491120}.
\end{proof}

\section{Additional Simulation Results}\label{app:add}

Simulations in this section are based on a simple setting where $X \sim 20 \times \text{unif}(0, 1)$ and $Y = 2X + \mathcal{N}(0, 1)$. The number of training observations $n = 500$. The model is an ensemble of decision trees.

Figure \ref{fig:ve_250} and \ref{fig:ve_400} displays variance estimation by IM, BM and IJ for kernel size $k_n = 250$ and $400$. 

\begin{figure}[H]
\centering
\begin{subfigure}{\textwidth}
\includegraphics[width=\textwidth]{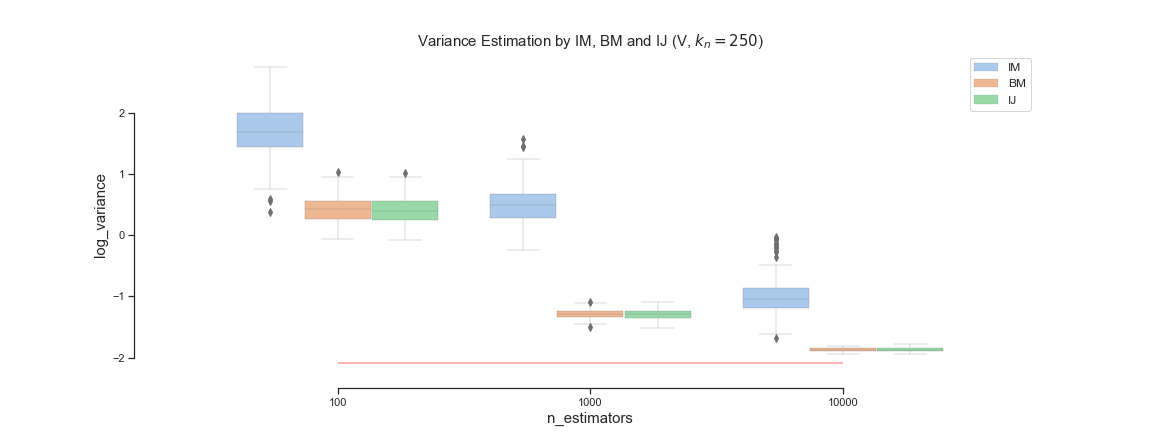}
\caption{Subsampling with replacement ($V$-statistics).}
\label{fig:ve_v_250}
\end{subfigure}
\begin{subfigure}{\textwidth}
\includegraphics[width=\textwidth]{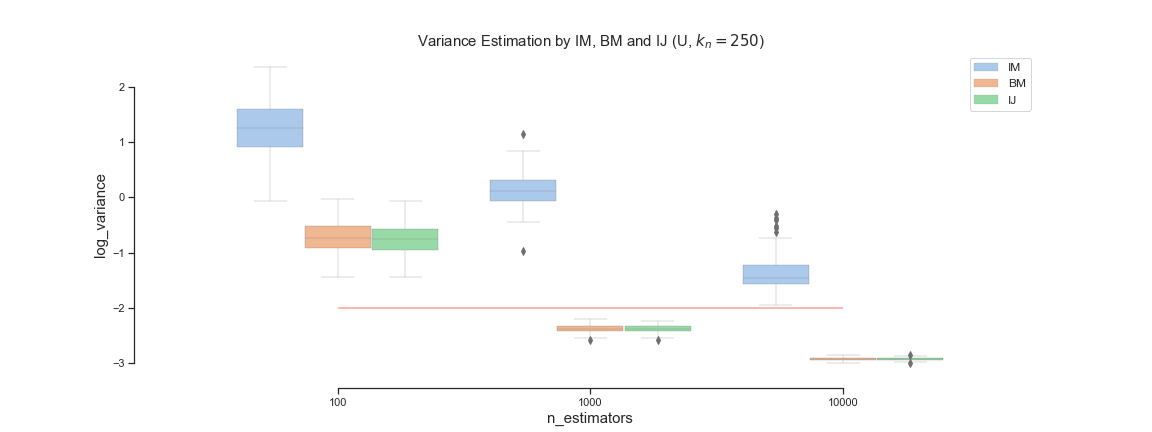}
\caption{Subsampling without replacement ($U$-statistics).}
\label{fig:ve_u_250}
\end{subfigure}

\caption{Variance estimation by three different methods: IM, BM and IJ. The kernel size $k_n = 250$. The variance shown is for prediction at test point $x = 10$. The red line denotes true (log) variance obtained by generating data, training the ensemble 100 times and calculating the empirical variance of predictions.}
\label{fig:ve_250}
\end{figure}

\begin{figure}[H]
\centering
\begin{subfigure}{\textwidth}
\includegraphics[width=\textwidth]{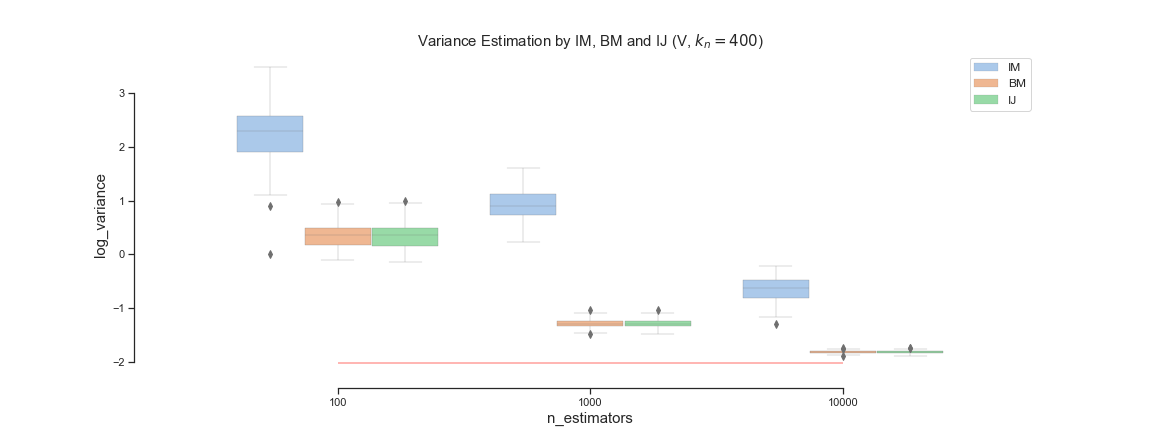}
\caption{Subsampling with replacement ($V$-statistics).}
\label{fig:ve_v_400}
\end{subfigure}
\begin{subfigure}{\textwidth}
\includegraphics[width=\textwidth]{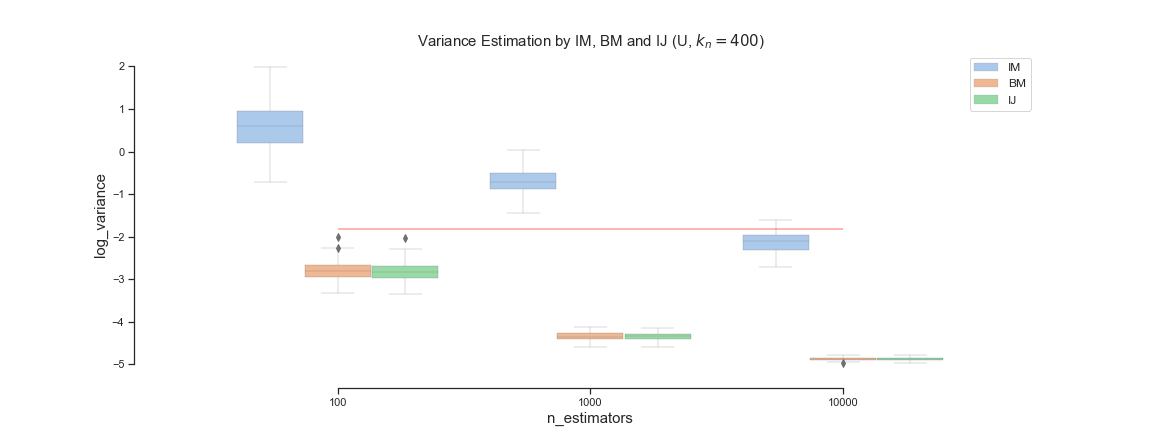}
\caption{Subsampling without replacement ($U$-statistics).}
\label{fig:ve_u_400}
\end{subfigure}

\caption{Variance estimation by three different methods: IM, BM and IJ. The kernel size $k_n = 400$. The variance shown is for prediction at test point $x = 10$. The red line denotes true (log) variance obtained by generating data, training the ensemble 100 times and calculating the empirical variance of predictions.}
\label{fig:ve_400}
\end{figure}

Figure \ref{fig:zeta1} and Figure \ref{fig:zetan} shows the estimated values for each variance components $\zeta_{1, k_n}$ and $\zeta_{k_n, k_n}$. Since IJ does not target at $\zeta_{1, k_n}$ directly, we rescaled the estimated by a factor $\frac{k_n^2}{n}$ according to Theorem \ref{thm:bm2ij}. The estimators for $\zeta_{k_n, k_n}$ for the three methods shown are essentially the same as they are all calculating the variance across all base learners' predictions.  

\begin{figure}[H]
\centering
\begin{subfigure}{\textwidth}
\includegraphics[width=\textwidth]{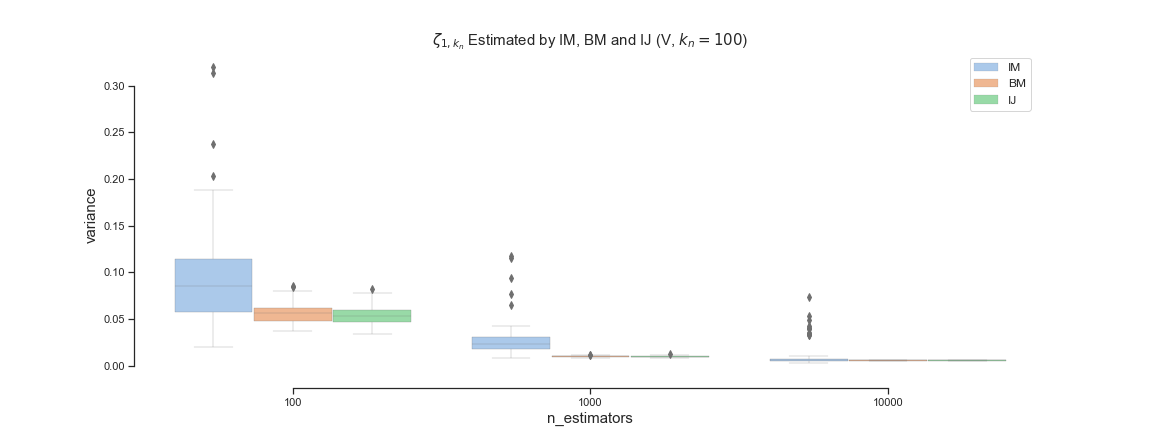}
\caption{Subsampling with replacement ($V$-statistics).}
\label{fig:ve_v_zeta1}
\end{subfigure}
\begin{subfigure}{\textwidth}
\includegraphics[width=\textwidth]{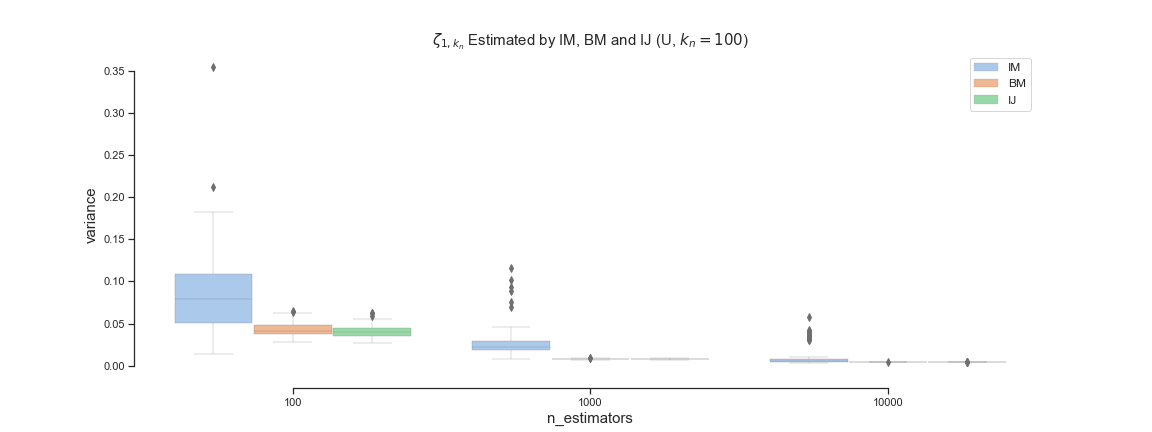}
\caption{Subsampling without replacement ($U$-statistics).}
\label{fig:ve_u_zeta1}
\end{subfigure}

\caption{$\zeta_{1,k_n}$ estimated by three different methods: IM, BM and IJ. The kernel size $k_n = 100$. The variance shown is for prediction at test point $x = 10$.} 
\label{fig:zeta1}
\end{figure}

\begin{figure}[H]
\centering
\begin{subfigure}{\textwidth}
\includegraphics[width=\textwidth]{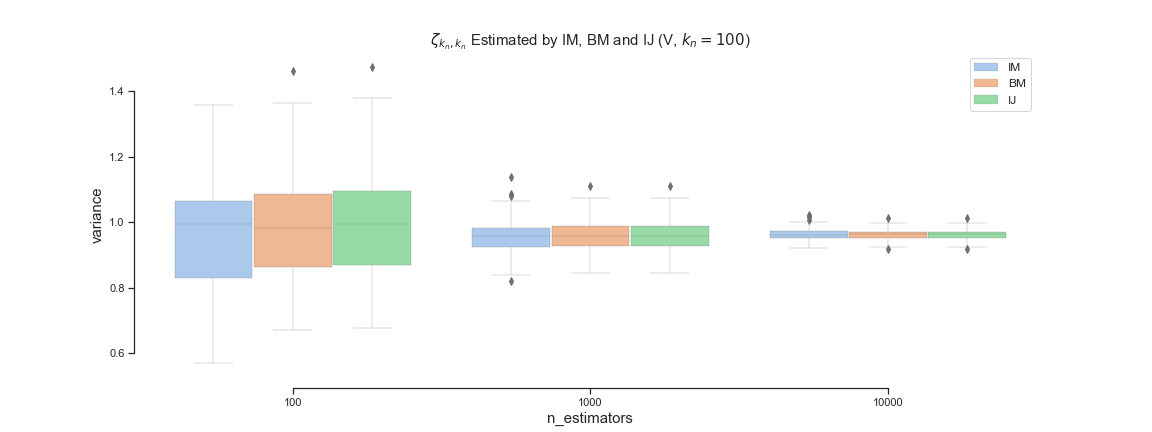}
\caption{Subsampling with replacement ($V$-statistics).}
\label{fig:ve_v_zetan}
\end{subfigure}
\begin{subfigure}{\textwidth}
\includegraphics[width=\textwidth]{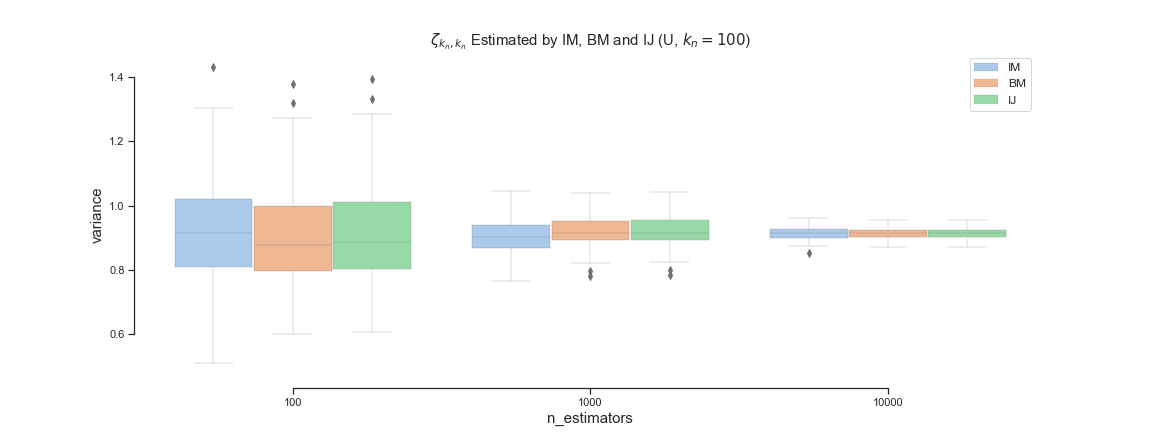}
\caption{Subsampling without replacement ($U$-statistics).}
\label{fig:ve_u_zetan}
\end{subfigure}

\caption{$\zeta_{k_n,k_n}$ estimated by three different methods: IM, BM and IJ. The kernel size $k_n = 100$. The variance shown is for prediction at test point $x = 10$.} 
\label{fig:zetan}
\end{figure}

\section{Derivations on Bias-corrected Estimator}\label{app:bc_dev}

Our goal is to provide an estimator of $\zeta_{1,k_n}$ based on expression given in (\ref{eqn:c_zeta1})
$$
    \zeta_{1,k_n} = \text{var} \left(\mathbb{E}\left(h_{k_n}\left(Z_1, \ldots, Z_{k_n}\right)|Z_1 = z_1\right)\right).
$$

To simplify notations, we introduce a more general mathematical representation. Consider a random variable $X$ and its conditional distribution given a random variable $Z$, denoted by $M = \mathbb{E}(X | Z)$. We want to estimate the variance of $\sigma_M^2 = \text{Var}(M)$.  Use $F_Z$ and $F_{X | Z}$ to denote the distribution for $Z$ and the conditional distribution $X$ given $Z$ respectively.

Consider the following sampling framework:
for $k = 1, \ldots, K$:
\begin{enumerate}
    \item Sample $Z_k$ randomly from $F_Z$.
    \item For $j = 1, \ldots, n_k$: Sample $X_{kj}$ randomly from $F_{X|Z = Z_k}$. 
\end{enumerate}
We'll use the collections of samples $X_{kj}$ $(k = 1, \ldots, K, j = 1, \ldots, n_k)$ to provide an estimator for $\sigma_M$. Define $C = \sum_{k = 1}^Kn_k$, $\sigma_{\epsilon}^2 = \mathbb{E}(\text{Var}(X|Z))$ and the following two sum of squares
$$
SS_{\tau} = \sum_{k = 1}^Kn_k(\bar{X}_k - \bar{\bar{X}})^2,
$$
$$
SS_{\epsilon} = \sum_{k=1}^K\sum_{j = 1}^{n_k}(X_{kj} - \bar{X}_k)^2
$$
where $\bar{\bar{X}} = \frac{1}{C}\sum_{k = 1}^Kn_k\bar{X}_k$, $\bar{X}_k = \frac{1}{n_k}\sum_{j = 1}^{n_k}X_{kj}$. Following the calculations in \cite{MR2844419}, we have
$$
\mathbb{E}(SS_{\tau}) = \left(C - \sum_{i = 1}^Kn_i^2/C\right)\sigma_M^2 + \left(K - 1\right)\sigma_{\epsilon}^2,
$$
$$
\mathbb{E}(SS_{\epsilon}) = (C - K)\sigma_{\epsilon}^2.
$$
Thus we can get the estimator for $\sigma_M^2$ as
$$
\hat{\sigma}_M^2 = \frac{SS_{\tau} - (K - 1)\hat{\sigma}_{\epsilon}^2}{C - \sum_{i = 1}^Kn_i^2 / C}
$$
where
$$
\hat{\sigma}_{\epsilon}^2 = \frac{SS_\epsilon}{C - K}.
$$
The unbiasedness of these estimators is shown by \cite{MR2298115}. By setting $Z = Z_1$ and $X = h_{k_n}\left(Z_1, \ldots, Z_{k_n}\right)$ gives the estimator presented in Section \ref{sec:bce}.

\section{An Alternative Version of Bias Correction}\label{app:simpler}

Our subsampling methods choose each data point with equal probability and we thus expect to obtain an approximately balanced subsample. For simplicity, our derivation assumes this structure holds exactly.  

Recall that we have $N_1 = N_2 = \ldots = N_i = r_n$, then 
$$
\hat{\sigma}_{\epsilon}^2 = \frac{SS_{\epsilon}}{C - n} = \frac{1}{n(r_n - 1)}\sum_{i=1}^n\sum_{b=1}^{B_n} N_{i, b}(h_b - m_i)^2,
$$
and 
$$
    \hat{\zeta}_{1,k_n} = \frac{1}{n-1}\sum_{i=1}^{n}(m_i - \bar{m})^2 - \frac{1}{r_n}\hat{\sigma}_{\epsilon}^2. 
$$

We could rewrite $\hat{\sigma}_{\epsilon}^2$ as

\begin{align*}
    \hat{\sigma}_{\epsilon}^2 &= \frac{1}{n(r_n - 1)}\sum_{i=1}^n\sum_{b=1}^{B_n} N_{i, b}(h_b - m_i)^2 \\
    &= \frac{1}{n(r_n - 1)}\sum_{i=1}^n\sum_{b=1}^{B_n} N_{i, b}(h_b - \bar{h} + \bar{h}- m_i)^2 \\
    &= \frac{1}{n(r_n - 1)} \sum_{i=1}^n\sum_{b=1}^{B_n}N_{i, b}(h_b - \bar{h})^2 + \frac{1}{n(r_n - 1)} \sum_{i=1}^n\sum_{b=1}^{B_n}N_{i, b}(\bar{h}- m_i)^2 \\
    &\quad + \frac{2}{n(r_n - 1)} \sum_{i=1}^n\sum_{b=1}^{B_n}N_{i, b}(h_b - \bar{h})(\bar{h}- m_i) \\
    &= \frac{1}{n(r_n - 1)} \sum_{b=1}^{B_n}(h_b - \bar{h})^2\sum_{i=1}^n N_{i, b} + \frac{1}{n(r_n - 1)} \sum_{i=1}^n (\bar{h}- m_i)^2  \sum_{b=1}^{B_n}N_{i, b} \\
    &\quad + \frac{2}{n(r_n - 1)} \sum_{i=1}^n (\bar{h}- m_i) \sum_{b=1}^{B_n}N_{i, b}(h_b - \bar{h}) \\
    &= \frac{1}{n(r_n - 1)} \sum_{b=1}^{B_n}(h_b - \bar{h})^2 k_n + \frac{1}{n(r_n - 1)} \sum_{i=1}^n (\bar{h}- m_i)^2  r_n \\
    &\quad + \frac{2}{n(r_n - 1)} \sum_{i=1}^n (\bar{h}- m_i) r_n(m_i - \bar{h}) \\
    &= \frac{k_n}{n(r_n - 1)} \sum_{b=1}^{B_n}(h_b - \bar{h})^2  - \frac{r_n}{n(r_n - 1)} \sum_{i=1}^n (\bar{h}- m_i)^2 \\
    &= \frac{k_n}{n(r_n - 1)} \sum_{b=1}^{B_n}(h_b - \bar{h})^2  - \frac{r_n}{n(r_n - 1)} \sum_{i=1}^n (m_i - \bar{m})^2. \\
\end{align*}

Plug this into the expression for $\hat{\zeta}_{1,k_n}$, we have

\begin{align*}
     \hat{\zeta}_{1,k_n} &= \frac{1}{n-1}\sum_{i=1}^{n}\left(m_i - \bar{m}\right)^2 - \frac{1}{r_n}\hat{\sigma}_{\epsilon}^2 \\
     &= \frac{1}{n-1}\sum_{i=1}^{n}\left(m_i - \bar{m}\right)^2 - \frac{1}{r_n}\left(\frac{k_n}{n(r_n - 1)} \sum_{b=1}^{B_n}\left(h_b - \bar{h}\right)^2  - \frac{r_n}{n(r_n - 1)} \sum_{i=1}^n \left(m_i - \bar{m}\right)^2\right) \\
     &= \left(\frac{1}{n - 1} - \frac{1}{n (r_n - 1)}\right)\sum_{i=1}^{n}\left(m_i - \bar{m}\right)^2 - \frac{k_n}{r_nn(r_n - 1)}\sum_{b=1}^{B_n}\left(h_b - \bar{h}\right)^2 \\
     &= \left(\frac{1}{n - 1} - \frac{1}{n (r_n - 1)}\right)\sum_{i=1}^{n}\left(m_i - \bar{m}\right)^2 - \frac{k_n}{r_nn(r_n - 1)}\left(B_n - 1\right)\hat{\zeta}_{k_n,k_n}^{\text{BM}} \\
     &\approx \frac{1}{n - 1}\sum_{i=1}^{n}\left(m_i - \bar{m}\right)^2 - \frac{1}{B_n}\frac{n}{k_n}\hat{\zeta}_{k_n,k_n}^{\text{BM}}.
\end{align*}

The approximation in the last line holds as long as $r_n$ grows with $n$, which is a reasonable assumption in most cases.

\section{Bias Correction for $U$-statistics}\label{app:u_cor}

Simulations in this section are based on a simple setting where $X \sim 20 \times \text{unif}(0, 1)$ and $Y = 2X + \mathcal{N}(0, 1)$. The number of training observations $n = 500$. The model is an ensemble of decision trees built under the framework of $U$-statistics: each tree is constructed using subsamples \emph{without} replacement.

Figure \ref{fig:u_wager} shows the result for $U$-statistics by employing the correction by \cite{MR3862353}.

\begin{figure}[H]
\centering
\begin{subfigure}{0.98\textwidth}
\includegraphics[width=\textwidth]{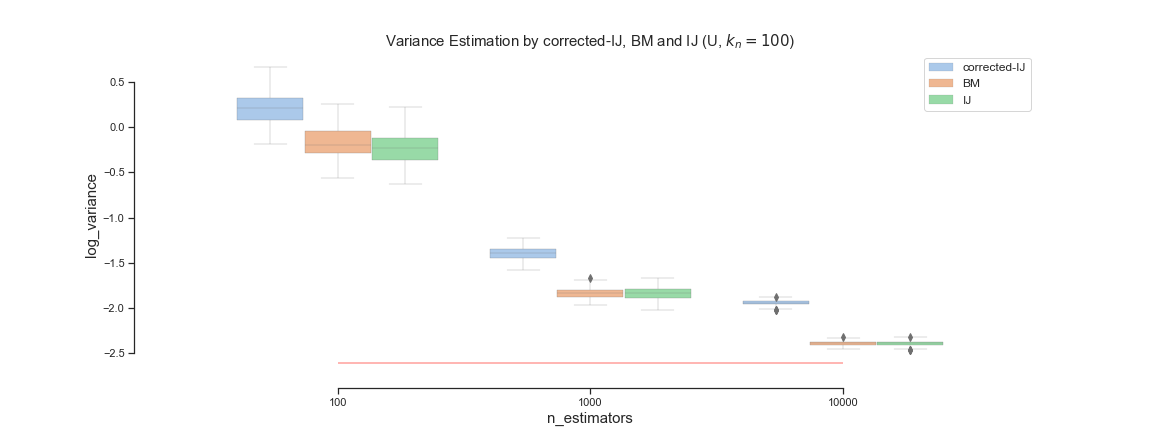}
\caption{$k_n = 100$.}
\label{fig:u_wager_100}
\end{subfigure}
\begin{subfigure}{\textwidth}
\includegraphics[width=\textwidth]{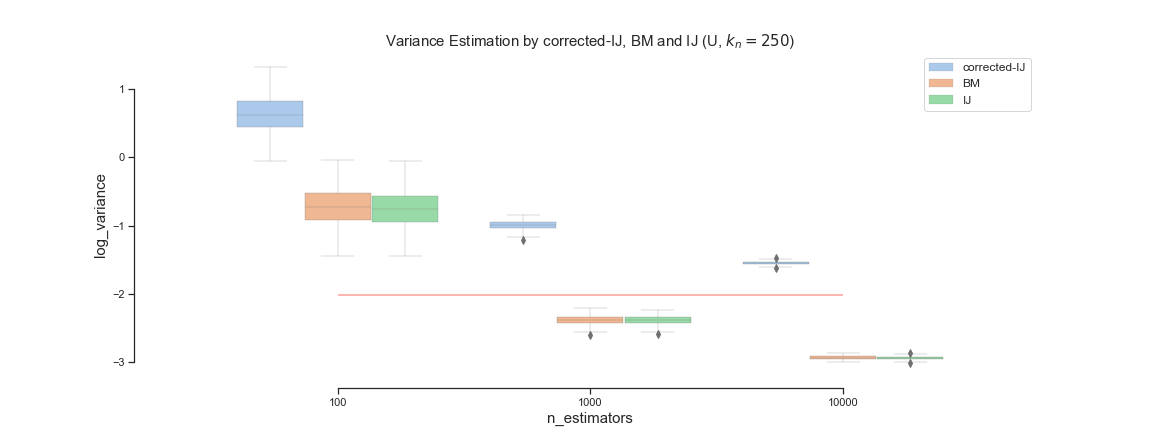}
\caption{$k_n = 250$.}
\label{fig:u_wager_250}
\end{subfigure}
\begin{subfigure}{\textwidth}
\includegraphics[width=\textwidth]{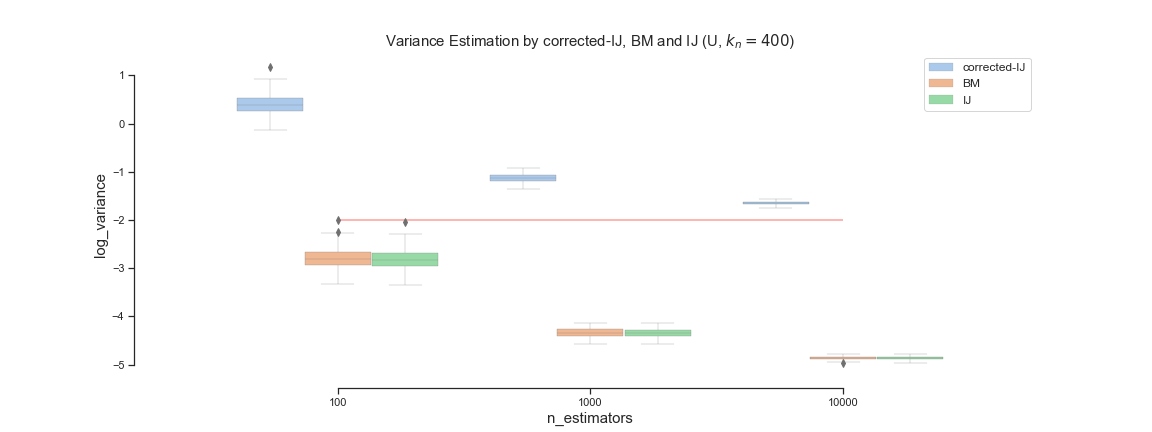}
\caption{$k_n = 400$.}
\label{fig:u_wager_400}
\end{subfigure}

\caption{Variance Estimation by three different methods: corrected-IJ, BM and IJ. The kernel size $k_n = 100, 250, 400$. The variance shown is for prediction at test point $x = 10$. The red line denotes true (log) variance obtained by generating data, training the ensemble 100 times and calculating the empirical variance of predictions.}
\label{fig:u_wager}
\end{figure}

Figure \ref{fig:ui_cor} shows the result of corrected-V developed in Section \ref{sec:bce} applied to $U$-statistics.

\begin{figure}[H]
    \centering
    \includegraphics[width=\textwidth]{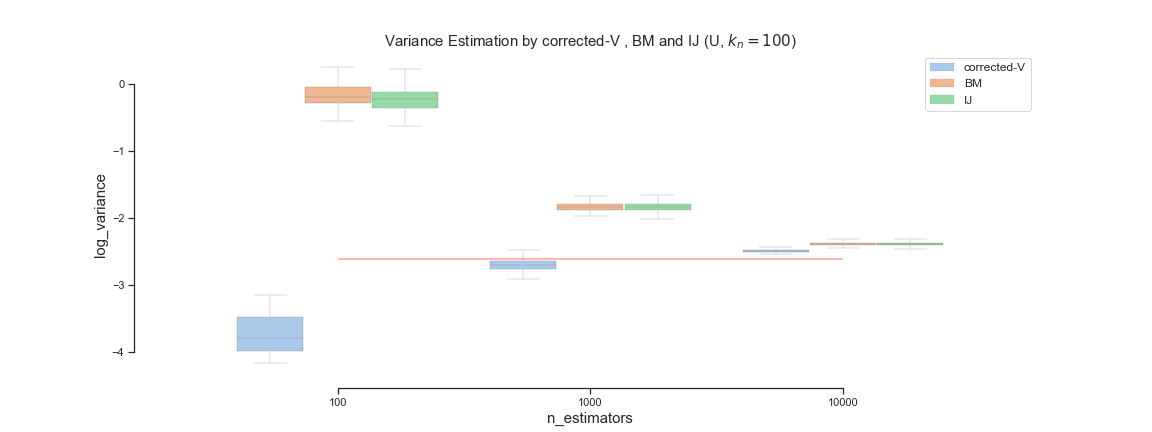}
    \caption{Variance estimation by three different methods: corrected-V, BM and IJ. The variance shown is for prediction at test point $x = 10$. The red line denotes true (log) variance obtained by generating data, training the ensemble 100 times and calculating the empirical variance of predictions.}
    \label{fig:ui_cor}
\end{figure}

For simplicity, we will use the simpler but approximate variance estimation described in Appendix \ref{app:simpler}
$$
\hat{\zeta}_{1,k_n} = \frac{1}{n - 1}\sum_{i=1}^{n}(m_i - \bar{m})^2 - \frac{1}{B_n}\frac{n}{k_n}\hat{\zeta}_{k_n,k_n}^{\text{BM}}.
$$
We find that if we scale the correction term by $\frac{n - k_n}{n}$, and include the correction term in \cite{MR3862353}, it works for $U$-statistics empirically. 
The estimator for $\zeta_{1,k_n}$ for $U$-statistics is
$$
\hat{\zeta}_{1,k_n}^U = \frac{n(n-1)}{(n - k_n)^2} \left( \frac{1}{n - 1}\sum_{i=1}^{n}(m_i - \bar{m})^2 - \frac{1}{B_n}\frac{n - k_n}{k_n}\hat{\zeta}_{k_n,k_n}^{\text{BM}}\right).
$$
The blue bars (denoted by corrected-U) in Figure \ref{fig:u_combined} shows the result. We can see that by combining both correction terms, the estimator yields stable and accurate variance estimation.

\begin{figure}[H]
\centering
\begin{subfigure}{\textwidth}
\includegraphics[width=\textwidth]{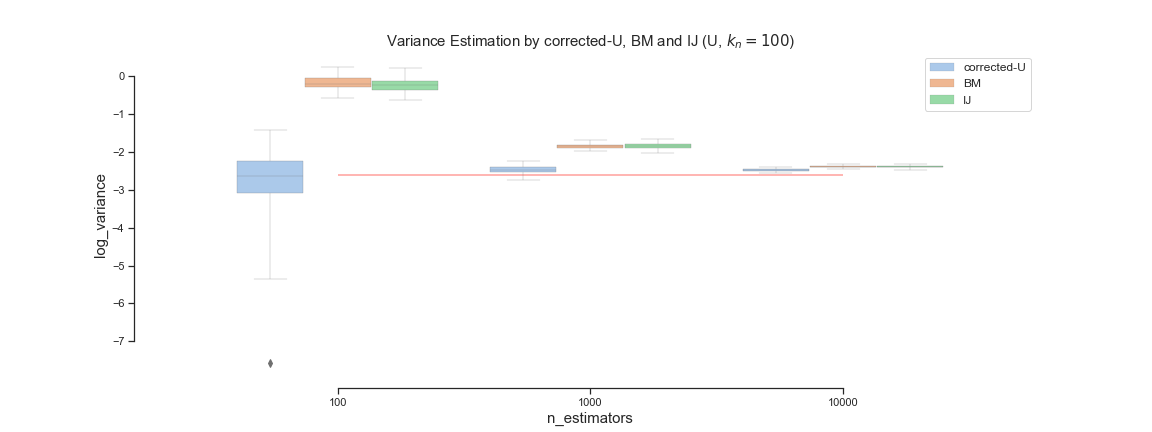}
\caption{$k_n = 100$.}
\label{fig:u_combined_100}
\end{subfigure}
\begin{subfigure}{\textwidth}
\includegraphics[width=\textwidth]{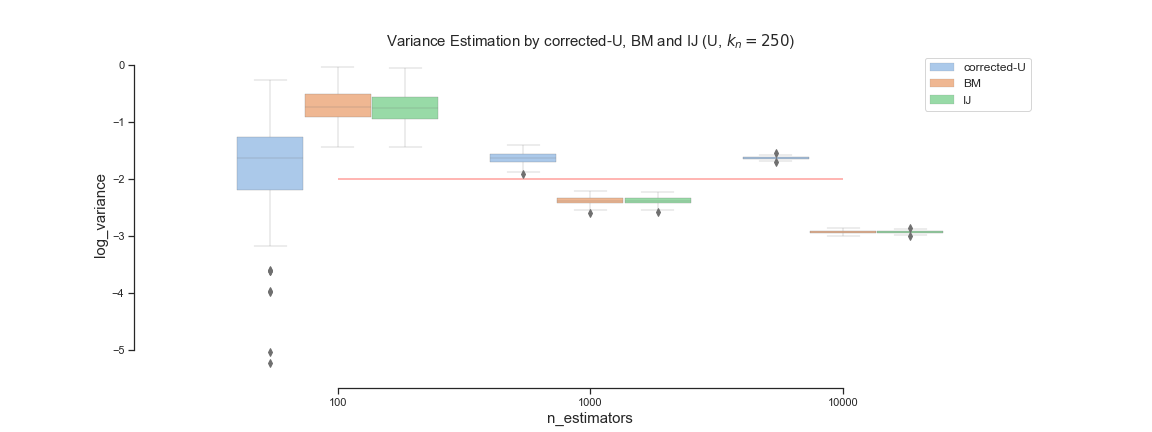}
\caption{$k_n = 250$.}
\label{fig:u_combined_250}
\end{subfigure}
\begin{subfigure}{\textwidth}
\includegraphics[width=\textwidth]{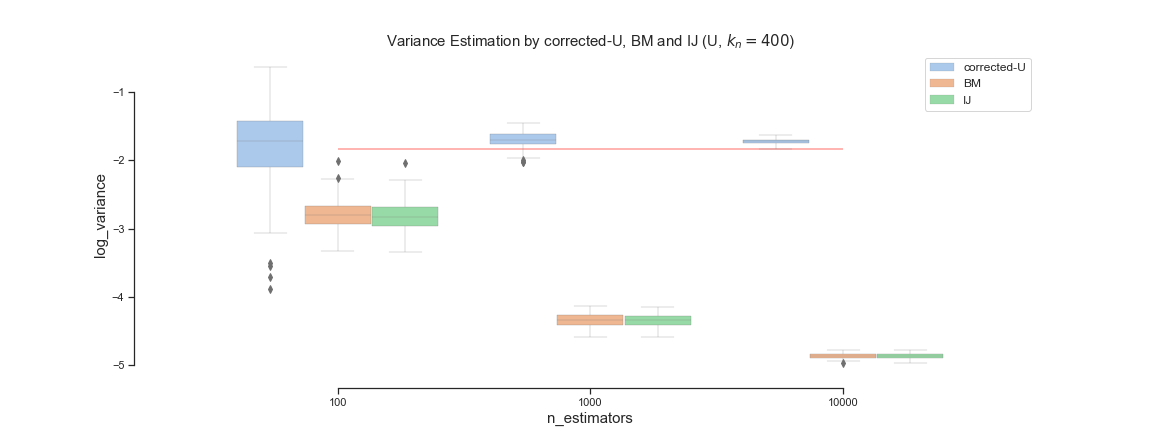}
\caption{$k_n = 400$.}
\label{fig:u_combined_400}
\end{subfigure}

\caption{Variance Estimation by three different methods: corrected-U, BM and IJ. The kernel size $k_n = 100, 250, 400$. The variance shown is for prediction at test point $x = 10$. The red line denotes true (log) variance obtained by generating data, training the ensemble 100 times and calculating the empirical variance of predictions.}
\label{fig:u_combined}
\end{figure}

\section{A Closer Look at Variance Components}

A major difference between our work and \cite{MR3862353} is that we also take into account the effect of Monte Carlo effect brought by the number of base learners $B_n$. Thus our variance has two components, where the first part $\frac{k_{n}^{2}}{n} \zeta_{1,k_n}$ corresponds to the complete case and the second part $\frac{1}{B_n} \zeta_{k_n,k_n}$ is the additional Monte Carlo variance introduced due to the incomplete case. 

One can imagine that for smaller $B_n$, the second part of Monte Carlo variance might be much larger than the first part, while as $B_n$ gets larger the effect diminishes and $\frac{k_{n}^{2}}{n} \zeta_{1,k_n}$ becomes the dominating one. We conduct an experiment to visualize this transition. As before, let $X \sim 20 \times \text{unif}(0, 1)$ and $Y = 2X + \mathcal{N}(0, 1)$. The model is an ensemble of decision trees built under the framework of $V$-statistics. We fix the number of training observations $n = 1000$ and kernel size $k_n = 10$ and build the ensembles with $B_n = 100, 200, 300, 400, 500 ,600, 700, 800, 900, 1000$. For each $B_n$, 100 models are built and we calculate empirical variance of the predictions at test point $x = 10$. This procedure is repeated 10 times and we report the average of empirical variance.

Figure \ref{fig:var_cp} shows four lines: the empirical variance; the two variance components $\frac{k_{n}^{2}}{n} \zeta_{1,k_n}$ and $\frac{1}{B_n} \zeta_{k_n,k_n}$; the total estimated variance which is simply the sum $\frac{k_{n}^{2}}{n} \zeta_{1,k_n} + \frac{1}{B_n} \zeta_{k_n,k_n}$. We estimate $\zeta_{1,k_n}$ and $\zeta_{k_n,k_n}$ using an ensemble of size $B_n = 1000$. The dotted black line aligns well with the black line, which indicates that our variance estimates give accurate results. For small $B_n = 100$, each observation is expected to only appear once in the ensemble (since $r_n = \frac{k_n \times B_n}{n} = 1$), and as a result base learners will be approximately independent. In this case, the variance of the ensemble prediction should mainly come from $\frac{1}{B_n} \zeta_{k_n,k_n}$. When $B_n$ grows, dependence between some base learners kicks in and the effect of $\frac{k_{n}^{2}}{n} \zeta_{1,k_n}$ gradually becomes the dominating part as $\frac{1}{B_n} \zeta_{k_n,k_n}$ decreases. This transition is depicted in Figure \ref{fig:var_cp}. Note that in practice the additional Monte Carlo variance introduced due to incomplete case is usually negligible as we would choose larger $k_n$ and $B_n$.

\begin{figure}[H]
    \centering
    \includegraphics[width=\textwidth]{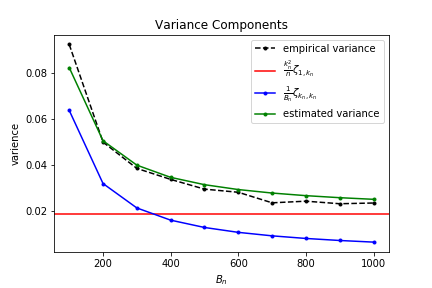}
    \caption{Variance components for different $B_n$. The number of training observations $n = 1000$ and kernel size $k_n = 10$. The variance shown is for prediction at test point $x = 10$. Four lines shown are: empirical variance, two variance components ($\frac{k_{n}^{2}}{n} \zeta_{1,k_n}$ and $\frac{1}{B_n} \zeta_{k_n,k_n}$) and their sum as estimated variance.}
    \label{fig:var_cp}
\end{figure}

\section{Datasets Information}\label{app:datasets}

Six of the seven datasets in Table \ref{table:pred} are taken from UCI Machine Learning Repository\footnote{\href{https://archive.ics.uci.edu/ml/index.php}{https://archive.ics.uci.edu/ml/index.php}}:

\begin{itemize}
    \item boston: \href{https://archive.ics.uci.edu/ml/machine-learning-databases/housing/}{https://archive.ics.uci.edu/ml/machine-learning-databases/housing/}. The dataset contains information collected by the U.S Census Service concerning housing in the area of Boston Mass. This is a regression task to predict median value of owner-occupied homes in \$1000's.
    \item diabetes: \href{https://archive.ics.uci.edu/ml/datasets/diabetes}{https://archive.ics.uci.edu/ml/datasets/diabetes}. The attributes are diabetes patient records and the target is an integer between 25 and 346. We simply cast it as a regression problem. 
    \item iris: \href{https://archive.ics.uci.edu/ml/datasets/Iris}{https://archive.ics.uci.edu/ml/datasets/Iris} This is a classification problem. The dataset contains 3 classes of 50 instances each, where each class refers to a type of iris plant.
    \item digits: \href{https://archive.ics.uci.edu/ml/datasets/optical+recognition+of+handwritten+digits}{https://archive.ics.uci.edu/ml/datasets/optical+recognition+of+handwritten+digits}. A classification task to predict integers from 0 to 9 with 64 attributes. 
    \item retinopathy: \href{https://archive.ics.uci.edu/ml/datasets/Diabetic+Retinopathy+Debrecen+Data+Set}{https://archive.ics.uci.edu/ml/datasets/Diabetic+Retinopathy+Debrecen+Data+Set}. This dataset contains features extracted from the Messidor image set to predict whether an image contains signs of diabetic retinopathy or not. 
    \item breast\_cancer: \href{https://archive.ics.uci.edu/ml/datasets/Breast+Cancer+Wisconsin+(Diagnostic)}{https://archive.ics.uci.edu/ml/datasets/Breast+Cancer+Wisconsin+(Diagnostic)}. This is a classification task to predict whether the diagnosis is malignant or benign based on features computed from a digitized image of a fine needle aspirate (FNA) of a breast mass.
\end{itemize}

\end{appendix}

\section*{Acknowledgements}
All three authors were partially supported by NSF grant DMS-1712554.

\bibliographystyle{imsart-number} 
\bibliography{bibliography}       





\end{document}